%% file: main.tex
\documentclass[10pt]{article} 

\usepackage[preprint]{rlj}           

\usepackage{amssymb}
\usepackage{amsthm}
\usepackage{mathtools}
\usepackage{booktabs}
\usepackage{nicefrac}
\usepackage{microtype}
\usepackage{graphicx}
\usepackage{subcaption}
\usepackage{amsfonts}
\usepackage{listings}
\usepackage{xcolor}
\usepackage{enumitem}

\lstdefinestyle{python}{
  language=Python,
  basicstyle=\ttfamily\tiny,
  keywordstyle=\color{blue}\bfseries,
  stringstyle=\color{red!70!black},
  commentstyle=\color{green!50!black}\itshape,
  numberstyle=\tiny\color{gray},
  numbers=left,
  breaklines=true,
  showstringspaces=false,
  morekeywords={jnp, True, False, None, self},
}

\tcbuselibrary{skins,breakable}

\definecolor{exampleBack}{RGB}{231,246,250}
\definecolor{exampleLine}{RGB}{118,209,225}

\newtcolorbox{examplebox}[1][]{%
  enhanced, breakable,
  colback=exampleBack, colbacktitle=exampleBack,
  coltitle=black,
  boxrule=0pt, frame hidden,
  arc=1.6mm,
  left=5mm, right=5mm, top=4mm, bottom=4mm,
  toptitle=3mm, bottomtitle=2mm,
  fontupper=\small,
  fonttitle=\small\bfseries,
  before upper={\setlength{\parskip}{4pt}\setlength{\parindent}{0pt}},
  overlay unbroken={\fill[exampleLine]
    ([xshift=1.6mm]frame.north west) rectangle ([xshift=1.6mm+3pt]frame.south west);},
  overlay first={\fill[exampleLine]
    ([xshift=1.6mm]frame.north west) rectangle ([xshift=1.6mm+3pt]frame.south west);},
  overlay middle={\fill[exampleLine]
    (frame.north west) rectangle ([xshift=3pt]frame.south west);},
  overlay last={\fill[exampleLine]
    (frame.north west) rectangle ([xshift=1.6mm+3pt]frame.south west);},
  #1
}

\newtheorem{theorem}{Theorem}
\newtheorem{lemma}[theorem]{Lemma}
\newtheorem{proposition}[theorem]{Proposition}

\title{Overcoming Valid Action Suppression in Unmasked Policy Gradient Algorithms}

\setrunningtitle{Overcoming Valid Action Suppression in Unmasked Policy Gradients}

\author{Renos Zabounidis$^{1}$, Roy Siegelmann$^{2}$, Mohamad Qadri$^{1}$, Woojun Kim$^{1}$, Simon Stepputtis$^{3}$, Katia P.\ Sycara$^{1}$}

\emails{renosz@andrew.cmu.edu}

\affiliations{
$^{1}$\textbf{Carnegie Mellon University} \quad
$^{2}$\textbf{Massachusetts Institute of Technology} \quad
$^{3}$\textbf{Virginia Tech, Department of Mechanical Engineering}
}

\contribution{We identify valid action suppression through shared parameters as a mechanism causing unmasked training to fail.}{When actions are invalid at visited states, policy gradients decrease their probabilities; shared parameters propagate this decrease to unvisited states where those actions are valid, causing exponential suppression before the agent reaches them. Rarely-valid actions such as descending staircases or opening doors suffer most severely. Prior theory only showed that masking preserves gradient correctness; we provide the first analysis of why unmasked training degrades performance through this mechanism.}

\contribution{We propose feasibility classification to learn validity-discriminating representations, enabling deployment without oracle masks.}{While masking prevents policy-level suppression, it maintains correlated representations because the encoder receives no gradient signal to distinguish valid from invalid states. Feasibility classification induces separate representations for valid and invalid states by training the encoder to predict action validity from observations. This enables a practical deployment strategy: train with masking for stability, then substitute the learned predictor when oracle masks are unavailable at test time.}

\contribution{We introduce KL-balanced classification that weights examples by their impact on policy behavior through KL divergence between oracle-masked and predicted-masked policies.}{Unlike uniform focal loss, KL-balanced prioritizes actions where misclassification would most change the policy. The weight for each action is the KL divergence between the policy using ground-truth masks and the policy using predicted masks. This improves deployment performance without oracle masks.}

\contribution{We demonstrate empirically that feasibility classification enables reliable deployment without oracle masks across multiple architectures and environments.}{On challenging domains with large action spaces, our approach maintains strong performance when deploying with learned predictors instead of ground-truth masks, while baseline masking collapses without oracle access. Feasibility classification reduces feature correlation between valid and invalid states, validating that it induces validity-discriminating representations.}


\summary{Action masking constrains policies to valid actions in discrete-action reinforcement learning. Existing theory proves that masking preserves gradient correctness, but does not explain why unmasked training fails. We identify valid action suppression as the mechanism: gradients at visited states propagate through shared parameters to suppress valid actions at unvisited states. We analyze this under linear parameterization, proving exponential suppression bounds when features remain correlated between visited and unvisited states. We validate empirically that deep networks exhibit this correlation through shared prefinal layer representations. We propose feasibility classification to resolve the deployment dilemma: training the encoder to predict validity enables deployment without oracle masks. We validate on Craftax and MiniHack with multiple architectures.}

\begin{document}

\makeCover
\maketitle

\begin{abstract}
In reinforcement learning environments with state-dependent action validity, action masking consistently outperforms penalty-based handling of invalid actions, yet existing theory only shows that masking preserves the policy gradient theorem. We identify a distinct failure mode of unmasked training: it systematically suppresses valid actions at states the agent has not yet visited. This occurs because gradients pushing down invalid actions at visited states propagate through shared network parameters to unvisited states where those actions are valid. We prove that for softmax policies with shared features, when an action is invalid at visited states but valid at an unvisited state $s^*$, the probability $\pi(a \mid s^*)$ is bounded by exponential decay due to parameter sharing and the zero-sum identity of softmax logits. This bound reveals that entropy regularization trades off between protecting valid actions and sample efficiency, a tradeoff that masking eliminates. We validate empirically that deep networks exhibit the feature alignment condition required for suppression, and experiments on Craftax, Craftax-Classic, and MiniHack confirm the predicted exponential suppression and demonstrate that feasibility classification enables deployment without oracle masks.
\end{abstract}

\input{sections/01_introduction}
\input{sections/02_related_work}
\input{sections/03_preliminaries}
\input{sections/04_theory}
\input{sections/05_method}
\input{sections/06_experiments}
\input{sections/07_conclusion}

\subsubsection*{Acknowledgments}
\label{sec:ack}
This material is based upon work supported by the National Science Foundation Graduate Research Fellowship Program under Grant No.\ DGE-2140739, the Defense Advanced Research Projects Agency (DARPA) under Contract No.\ FA8750-23-2-1015 (ANSR), and the Office of Naval Research under Grant No.\ N00014-23-1-2840. Any opinions, findings, and conclusions or recommendations expressed in this material are those of the author(s) and do not necessarily reflect the views of the National Science Foundation, DARPA, the Office of Naval Research, or the U.S.\ Government.

\bibliography{references}
\bibliographystyle{rlj}

\beginSupplementaryMaterials

\appendix

\input{appendix_sections/appendix_proofs}
\input{appendix_sections/appendix_environments}
\input{appendix_sections/appendix_validity}
\input{appendix_sections/appendix_training}
\input{appendix_sections/appendix_ablations}
\input{appendix_sections/appendix_additional_experiments}

\end{document}

%% file: sections/01_introduction.tex
\section{Introduction}
\label{sec:introduction}

In discrete-action reinforcement learning, agents often face state-dependent constraints where the set of valid actions varies by state. Action masking, which constrains the policy to valid actions by zeroing invalid action probabilities, is standard in robotic assembly, combinatorial optimization, and strategy games \citep{huang2020closer, Tang2020ImplementingAM, Stappert2025IntegratingHK, liu2025physicsawarecombinatorialassemblysequence}. Masking outperforms penalty-based approaches by margins exceeding $50\%$ \citep{hou2023exploring}.

Existing theory only proves that masking preserves gradient correctness \citep{huang2020closer}. Why unmasked training fails remains unexplained. This gap leaves practitioners without guidance on which actions matter most for masking. Furthermore, masking requires a validity oracle at every timestep, yet many deployment settings lack ground-truth validity functions. Training with masks and deploying without them fails because the policy learns no mechanism to infer validity.

Unmasked training fails through valid action suppression. When an action is invalid at visited states, gradients decrease its probability, and shared parameters propagate this decrease to unvisited states where the action is valid. Suppression is exponential and most significant for rarely-valid actions. These are applicable at few states yet critical for task completion, such as descending staircases or opening doors. These actions become exponentially suppressed before the agent reaches the states where they are needed (Figure~\ref{fig:suppression_mechanism}).

Masking prevents suppression by excluding invalid actions from the policy entirely. However, when invalid actions are never considered, the encoder learns no features for distinguishing valid from invalid states. We propose feasibility classification to restore this learning signal. Training the encoder to predict validity from observations induces validity-discriminating features. Classification enables deployment without oracle masks: train with masking for stability, then substitute the learned predictor at test time. Practitioners should prioritize classification accuracy on rarely-valid actions. Because suppression grows exponentially with training steps (Theorem~\ref{thm:prob_suppression}), actions that remain invalid throughout most of training accumulate the strongest suppression. Yet these same actions---descending staircases, opening doors, using special abilities---are often the only path to task completion. A classifier that errs on common actions such as movement directions has minimal impact; one that errs on rare critical actions reintroduces the suppression problem at deployment.

Our contributions are: (i) a theoretical characterization of valid action suppression, proving exponential decay under invalid-action dominance and feature alignment (Theorem~\ref{thm:prob_suppression}); (ii) empirical evidence that oracle masking maintains high feature correlation near 0.8 while classification reduces it to approximately 0.4, restoring the gradient signal for learning validity-discriminating representations (Section~\ref{sec:rq2}); (iii) a KL-balanced classification loss that outperforms focal loss by $2\times$ when deploying without oracle masks (Section~\ref{sec:rq4}); and (iv) feasibility classification enabling deployment without oracle masks at 2\% performance cost (43.2 vs.\ 43.9) by substituting the learned predictor for ground-truth masks (Section~\ref{sec:rq4}).

%% file: sections/02_related_work.tex
\section{Related Work}
\label{sec:related_work}

\textbf{Action Masking.}
\citet{huang2020closer} proved that masking invalid actions before the softmax preserves unbiased policy gradients. They demonstrated gains over penalty-based handling. \citet{Tang2020ImplementingAM} implemented masking into PPO. \citet{hou2023exploring} systematically compared masked on-policy and off-policy algorithms in large discrete action spaces. Large action spaces themselves pose significant optimization and exploration challenges \citep{dulac2015deep}, motivating action-space pruning strategies such as masking. \citet{Stappert2025IntegratingHK} integrated domain knowledge via masking in operations research. \citet{liu2025physicsawarecombinatorialassemblysequence} derived physics-based masks for robotic assembly without additional environment interaction. Across domains including RTS games, combinatorial optimization, and robotics, masking is now standard practice. However, existing theory focuses on gradient correctness or empirical performance. None explains why \emph{unmasked} training degrades the policy at unvisited states. We provide the first analysis of the dynamics underlying this degradation.

\textbf{Action Elimination and Learned Masks.}
\citet{zahavy2019learnlearnactionelimination} introduced action elimination networks that learn to suppress suboptimal actions in DQN. \citet{Zhong2023NoPM} remove redundant actions via KL-based similarity. \citet{Wang2024LearningSA} learn state-specific masks using bisimulation metrics \citep{ferns2011bisimulation}. Similarly, \citet{zabounidis2025scalar} use LLM planning to compose and select skills, effectively pruning the action space for RL. These approaches treat masking as a learned compression of the action space. We analyze the dynamics when masking is absent or imperfect. Existing elimination methods do not quantify the cost of mask errors or characterize how invalid-action gradients propagate through shared representations. Our suppression theorem formalizes this cost and provides a framework for evaluating approximate or learned masks.

\textbf{Softmax Policy Gradient Theory.}
Prior work analyzes softmax convergence assuming all actions are admissible \citep{Mei2020OnTG, li2022softmaxpolicygradientmethods, cen2021fastglobalconvergencenatural}. We show that when some actions are invalid at visited states, their gradients propagate through shared features to suppress valid actions elsewhere.

\textbf{Representation Interference in RL.}
\citet{Liu2023MeasuringAM} measured gradient interference across state-action pairs and proposed mitigation strategies. \citet{Lan2022OnTG} analyzed representation generalization in RL. \citet{zabounidis2026disentangled} study disentangled representations for improving generalization. \citet{nikishin2022primacybiasdeepreinforcement} documented a primacy bias in which early experience dominates later learning. Related work on loss of plasticity \citep{Abbas2023LossOP} and dormant neurons \citep{Sokar2023TheDN} highlight gradient flow and cross-task interference. Research on temporal credit assignment, such as RUDDER \citep{arjona2019rudder}, shows how gradients propagate backward across states and time. These works document interference empirically. We identify a structurally induced and analytically tractable form of interference. Invalid-action gradients at visited states propagate through shared features and exponentially suppress valid actions at unvisited states. Our theorem makes this cross-state coupling explicit.

\textbf{Constrained and Structured Action Spaces.}
Constrained policy optimization \citep{achiam2017constrained} formulates safety and feasibility as cMDPs solved via trust-region methods. \citet{stolz2024excluding} extend masking ideas to continuous actions, and \citet{dai2026languagemovementprimitivesgrounding} ground language models in structured robot motion primitives. \citet{theile2024action} show that with a perfect feasibility model, a CMDP reduces to an unconstrained MDP. \citet{Keane2025StrategyMA} analyze masking in value-based RL and prove convergence under contraction assumptions. Structured action-value architectures such as TreeQN and ATreeC \citep{farquhar2018treeqn} incorporate inductive bias over action structure to improve generalization. These approaches assume accurate feasibility information or constraint models. We complement this literature by analyzing the opposite regime: where invalid actions are not excluded. We show that in this case, shared function approximation can induce systematic policy degradation even before constraint violations occur.

%% file: sections/03_preliminaries.tex
\section{Preliminaries}
\label{sec:preliminaries}

Consider a Markov Decision Process (MDP) $(\mathcal{S}, \mathcal{A}, P, R, \gamma)$ with state space $\mathcal{S}$, action space $\mathcal{A}$, transition kernel $P(s' \mid s, a)$, reward function $R: \mathcal{S} \times \mathcal{A} \to \mathbb{R}$, and discount $\gamma \in (0,1)$.

\paragraph{Abstraction for reward and validity.}
For many complex tasks, reward functions are defined over high-level symbolic properties rather than raw observations. Achievement criteria such as ``craft an iron pickaxe'' are naturally expressed over symbolic properties like inventory contents rather than pixels. This implicitly assumes an \emph{abstraction function} $\chi: \mathcal{S} \to \mathcal{Z}$ mapping environment states to \emph{abstract states} $z \in \mathcal{Z}$, e.g.,
\[
\chi(s) = \{\texttt{has}(\texttt{wood}, 5),\; \texttt{placed}(\texttt{table}),\; \texttt{on}(\texttt{staircase}),\; \ldots\}
\]
where predicates like \texttt{has}, \texttt{placed}, and \texttt{on} capture task-relevant properties.

The same abstraction enables symbolic specification of action validity. Consider what makes an action valid: to \textsf{descend} stairs, the agent must be \texttt{on(staircase)}; to \textsf{open\_door}, it must be \texttt{adjacent(closed\_door)}. An action is \textit{valid} at state $s$ when executing it is both \textit{applicable} (preconditions satisfied in $\chi(s)$) and \textit{safe} (no catastrophic outcomes such as entering lava or violating hard constraints). Otherwise it is \textit{invalid}. We define the \emph{validity indicator} $\nu: \mathcal{Z} \times \mathcal{A} \to \{0, 1\}$ where $\nu(z, a) = 1$ iff action $a$ is valid at abstract state $z$. For state $s$, we write $\nu(s, a) \equiv \nu(\chi(s), a)$.

This abstraction is crucial because it means validity---like reward---can be computed exactly from the symbolic state without learning from pixels. The agent need not discover through trial and error that \textsf{descend} requires stairs.

Validity is distinct from optimality. Many valid actions are suboptimal---the agent must discover which through experience. Invalid actions are different: from the symbolic structure of the environment, we can determine they cannot accomplish anything useful. Crafting without materials, opening a door when none is adjacent, descending stairs that do not exist---these are dominated by at least one valid alternative. The environment may handle them as no-ops, penalties, or blocked transitions \citep{Khetarpal2020WhatCI}, but in every case the action is strictly worse than acting on a satisfied precondition. This \emph{dominance gap}, the margin by which invalid actions are worse than valid alternatives, is precisely Condition~(i) of Theorem~\ref{thm:prob_suppression}.

\paragraph{Visited and unvisited states.}
Under policy $\pi$ with visitation distribution $d_\pi(s) = (1 - \gamma) \sum_{t=0}^\infty \gamma^t \Pr(s_t = s \mid \pi)$, we partition the state space:
\begin{equation}
    \mathcal{S}_{\mathrm{vis}} = \{s : d_\pi(s) > 0\}, \qquad \mathcal{S}_{\mathrm{unvis}} = \mathcal{S} \setminus \mathcal{S}_{\mathrm{vis}}.
\end{equation}
Policy gradient updates carry no direct signal for $s' \in \mathcal{S}_{\mathrm{unvis}}$; all changes to $\pi(\cdot \mid s')$ arise indirectly through parameter coupling with $\mathcal{S}_{\mathrm{vis}}$. This coupling creates the suppression mechanism we analyze in Section~\ref{sec:theory}.

\paragraph{Softmax policies with shared parameters.}
Policies are parameterized by logits $z_a(s; \theta)$ with $\pi_\theta(a \mid s) \propto \exp(z_a(s; \theta))$. Our analysis applies to policies where logits are linear in a state feature representation: $z_a(s) = \phi(s)^\top w_a$ where $\phi: \mathcal{S} \to \mathbb{R}^d$ is the feature map and $w_a \in \mathbb{R}^d$ are per-action parameters. For neural network policies, $\phi(s)$ corresponds to prefinal layer activations---the representation shared across action heads. Section~\ref{sec:rq2} validates this empirically by measuring feature correlation in this layer. The action space size is $n = |\mathcal{A}|$. At initialization, $z_a(s; \theta_0) \approx c$ for all $a, s$, giving $\pi_{\theta_0}(a \mid s) \approx 1/n$.

\paragraph{Action masking.}
Given the validity function $\nu(s, a)$, \emph{action masking} restricts the policy to valid actions by excluding invalid actions from the softmax:
\begin{equation}
    \pi_\theta^{\mathrm{oracle}}(a \mid s) = \frac{\exp(z_a(s; \theta)) \cdot \nu(s, a)}{\sum_{j \in \mathcal{A}} \exp(z_j(s; \theta)) \cdot \nu(s, j)}.
\end{equation}
Invalid actions receive zero probability and contribute zero gradient, decoupling $\mathcal{S}_{\mathrm{vis}}$ from $\mathcal{S}_{\mathrm{unvis}}$ for actions that are invalid in the former but valid in the latter \citep{huang2020closer, Chandak2019ReinforcementLW}.

%% file: sections/04_theory.tex
\section{Gradient Dynamics of the Unmasked Policy}
\label{sec:theory}

Without masking, parameter coupling between $\mathcal{S}_{\mathrm{vis}}$ and $\mathcal{S}_{\mathrm{unvis}}$ causes gradient propagation that suppresses valid actions. When action $a$ is invalid at visited states but valid at unvisited $s^*$, repeated negative updates at visited states propagate through shared parameters. This suppresses $\pi(a \mid s^*)$ before the agent reaches $s^*$ (Figure~\ref{fig:suppression_mechanism}). We quantify this suppression and show it occurs exponentially under two conditions.

\begin{figure}[t!]
    \centering
    \includegraphics[width=\textwidth]{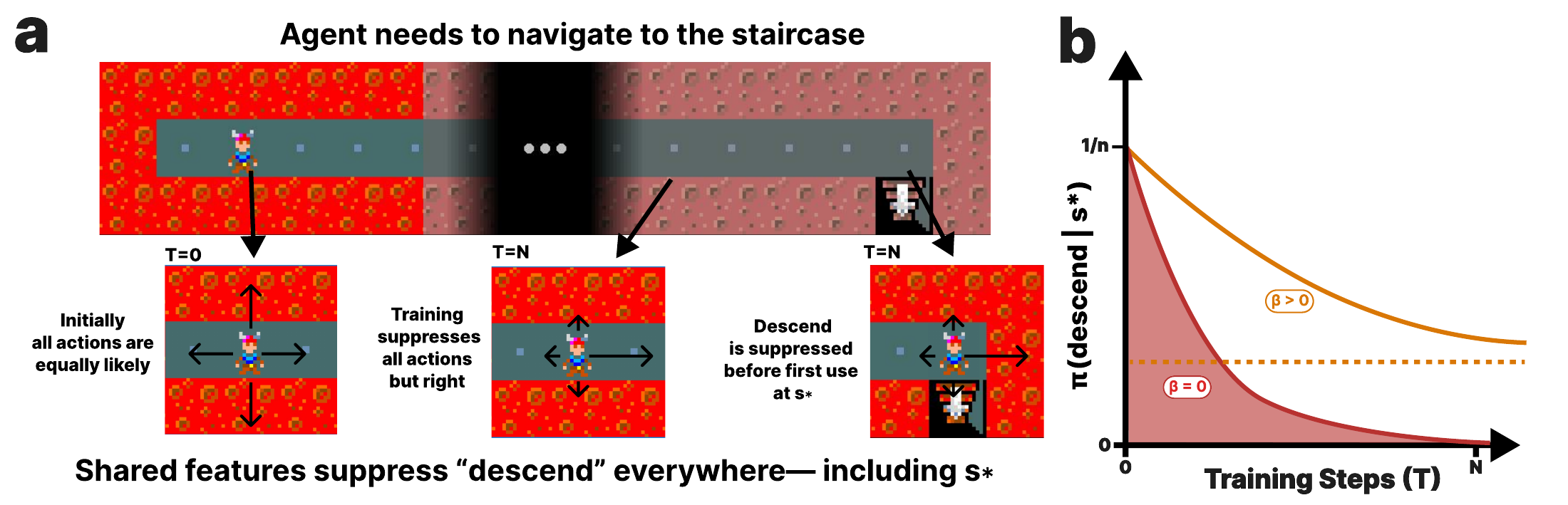}
    \caption{Suppression mechanism illustrated on a staircase corridor. \textbf{(a)} The agent must traverse a corridor to reach a staircase at $s^*$, where \textsf{descend} is the only valid goal-reaching action. At $T\!=\!0$, all actions are equally likely ($1/n$). After $N$ training steps, gradient updates at visited states reinforce \textsf{right} and suppress all other actions, including \textsf{descend}. Because parameters are shared, this suppression propagates to $s^*$ before the agent arrives. The \textsf{descend} action is suppressed at the one state where it is needed. \textbf{(b)} Upper bound on $\pi(\textsf{descend} \mid s^*)$ from Theorem~\ref{thm:prob_suppression}. Without entropy regularization ($\beta\!=\!0$), the probability decays exponentially. With entropy regularization ($\beta\!>\!0$), a floor emerges but cannot eliminate suppression (Eq.~\ref{eq:entropy_sandwich}).}
    \label{fig:suppression_mechanism}
\end{figure}

\paragraph{Conditions for suppression.}
Appendix~\ref{app:proofs} contains supporting lemmas.

\emph{(i) Invalid-action dominance gap.} At every visited state, invalid actions are strictly suboptimal whenever they appear in the policy support. At each step $\tau$ and state $s \in \mathcal{S}_{\mathrm{vis}}$:
\begin{equation}\label{eq:dominance}
    Q^{\pi_\tau}(s, a) \leq \max_{b : \nu(s, b) = 1} Q^{\pi_\tau}(s, b) - \delta_\tau(s)
\end{equation}
where $a$ is an invalid action and $\delta_\tau(s) > 0$ is some margin. Invalid actions cannot advance toward the goal and are dominated by at least one valid alternative. Under this condition, the policy gradient applies a strictly negative update to the logit of $a$ at every visited state by Lemma~\ref{lem:dominance_gap}.

\emph{(ii) Feature alignment.} The logit decrease at visited states propagates to unvisited state $s^*$ through the shared features $\phi(\cdot)$ whenever:
\begin{equation}\label{eq:non_orthogonal}
    \phi(s^*)^\top \mathbb{E}_{s \sim d_{\pi_\tau}}\!\left[\, c_\tau(s)\, \phi(s)\,\right] > 0,
\end{equation}
where $c_\tau(s) = \pi_\tau(a \mid s)\, |A^{\pi_\tau}(s, a)| > 0$ as shown in Proposition~\ref{prop:suppression}. Condition~(ii) holds when $\phi(s^*)$ is not orthogonal to the advantage-weighted mean of visited features. For neural networks, $\phi(s)$ is the prefinal layer activation. Condition~(ii) is satisfied when this representation does not perfectly isolate $s^*$ from visited states. Section~\ref{sec:rq2} measures this correlation directly.

\begin{theorem}[Probability suppression at first valid occurrence]\label{thm:prob_suppression}
Fix action $a$ and let $s^*$ be a state where $a$ is valid but $s^* \notin \mathcal{S}_{\mathrm{vis}}$. This is a \emph{first valid occurrence}. Assume uniform initialization $\pi_0(j \mid s^*) = 1/n$. Suppose Conditions~(i) and (ii) hold throughout training prior to visiting $s^*$. Let $\beta \geq 0$ be the entropy regularization coefficient. Define the per-step suppression rate:
\begin{equation}\label{eq:kappa_def}
    \kappa_\tau \;=\; \eta\, \phi(s^*)^\top \mathbb{E}_{s \sim d_{\pi_\tau}}\!\left[\, c_\tau(s)\, \phi(s)\,\right] \;>\; 0
\end{equation}
and the cumulative suppression $K_T = \sum_{\tau=0}^{T-1} \kappa_\tau$. Then after $T$ gradient steps before $s^*$ is first visited:
\begin{equation}\label{eq:T_step_decay}
    \pi_T(a \mid s^*) \leq \frac{e^{-K_T}}{n}.
\end{equation}
If $\beta > 0$, the cumulative suppression converges with $K_\infty < \infty$ and:
\begin{equation}\label{eq:entropy_sandwich}
    \exp\!\left(-\frac{r_{\max}}{\beta(1-\gamma)}\right) \;\leq\; \pi_\infty(a \mid s^*) \;\leq\; \frac{e^{-K_\infty}}{n}.
\end{equation}
\end{theorem}

The bound connects to neural networks through Condition~(ii): when the prefinal layer representation $\phi(s)$ remains correlated between visited and unvisited states, gradient updates propagate through shared weights. Section~\ref{sec:rq2} validates empirically that deep networks exhibit correlation of 0.4--0.8 throughout training, confirming the bound is tight in practice. Proof appears in Appendix~\ref{app:proofs}, with the sandwich bound \eqref{eq:entropy_sandwich} proven in Appendix~\ref{app:entropy_proof}.

Action masking sidesteps suppression by zeroing out invalid actions at the policy level. This decouples updates at visited states from probability mass at unvisited states. However, masking leaves the encoder's internal representation unchanged. The features $\phi(s)$ remain correlated between valid and invalid states because the encoder receives no gradient signal to distinguish them. This creates a deployment limitation: without oracle masks at test time, the policy has no mechanism to determine which actions are valid. To enable deployment without oracle masks, we must break Condition~(ii) at its source: train the encoder to predict validity from observations, reducing feature correlation between states where an action is valid versus invalid. This feasibility classification approach provides validity-aware representations while maintaining the training stability of oracle masking.

%% file: sections/05_method.tex
\section{Mitigating Suppression via Representation Learning}
\label{sec:method}

\begin{figure}[t!]
    \centering
    \includegraphics[width=\textwidth, height=6.75cm]{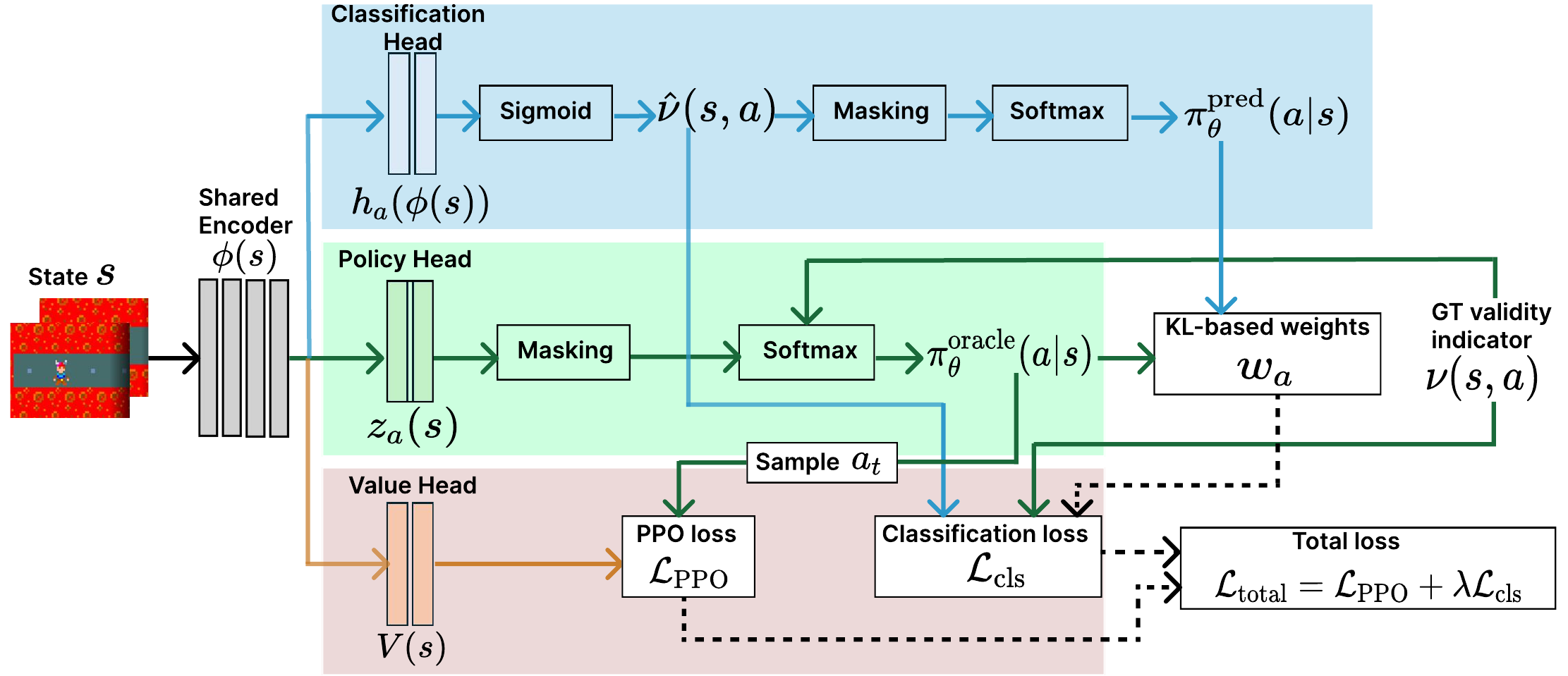}
    \caption{Architecture overview. A shared encoder $\phi(s)$ feeds three heads: a classification head (blue) predicting action validity $\hat{\nu}(s,a)$ via sigmoid, a policy head (green) producing oracle-masked actions $\pi^{\text{oracle}}_\theta(a \mid s)$, and a value head (orange) estimating $V(s)$. The predicted validity constructs a predicted policy $\pi^{\text{pred}}_\theta(a \mid s)$. The KL divergence between $\pi^{\text{oracle}}$ and $\pi^{\text{pred}}$ yields per-action weights $w_a$ for the classification loss (Eq.~\ref{eq:kl_weight}). The total loss combines PPO and weighted classification objectives (Eq.~\ref{eq:total_loss}). Dashed arrows indicate loss aggregation. At deployment, there are two modes: (i) with oracle masks available, classification heads can be discarded; (ii) without oracle masks, the learned predictor provides validity estimates for deployment.}
    \label{fig:architecture}
\end{figure}

Condition~(ii) states that suppression occurs when the representation $\phi(s^*)$ aligns with visited states. To break this alignment, we need features that distinguish states where an action is valid from states where it is invalid. The key insight is that predicting validity provides self-supervision. Learning to recognize when \textsf{descend} is valid at a staircase requires developing features that detect staircases, which naturally differ from corridor features where \textsf{descend} is invalid. This auxiliary task reshapes the representation to reduce correlation between valid and invalid states without requiring visitation of $s^*$ during training.

\paragraph{Feasibility Classification.}
Figure~\ref{fig:architecture} shows the architecture. We augment the policy network with lightweight classification heads that predict, for each action, whether that action is valid in the current state. For each action $a \in \mathcal{A}$, a binary classification head $h_a$ sits on top of the shared encoder $\phi$:
\begin{equation}
    \hat{\nu}(s, a) = \sigma\!\bigl(h_a(\phi(s))\bigr) \approx \nu(s, a),
\end{equation}
where $\sigma$ denotes the sigmoid. The encoder $\phi$ is shared with the policy and value networks, so gradients from the classification loss directly shape the representation. The heads $\{h_a\}_{a \in \mathcal{A}}$ are independent linear layers.

During training, we use ground-truth validity labels $\nu(s, a)$ from the same validity function used for masking (Section~\ref{sec:preliminaries}), evaluated on states drawn from the replay buffer $\mathcal{D}$. The classification loss augments the standard PPO objective:
\begin{equation}\label{eq:total_loss}
    \mathcal{L}_{\text{total}} = \mathcal{L}_{\text{PPO}} + \lambda \cdot \mathcal{L}_{\text{cls}},
\end{equation}
where $\lambda > 0$ controls the auxiliary strength.

This setup enables two deployment modes. When oracle masks are available at test time, the classification heads can be discarded and the policy uses ground-truth masking as usual. When oracle masks are unavailable, the learned predictor substitutes for them. The classification heads output $\hat{\nu}(s,a)$, which thresholded at $\tau$ (we use $\tau = 0.5$) provides predicted masks for deployment (Section~\ref{sec:rq4}).

\paragraph{Classification Objectives.}
Valid actions are rare: at any given state, only a small fraction of the action space is typically valid. Standard cross-entropy loss wastes gradient on the easy negatives (clearly invalid actions) and struggles with the hard negatives (edge cases where the classifier is uncertain). We develop two objectives that progressively focus learning on what matters for the policy.

\textbf{Focal loss} \citep{lin2017focal} addresses class imbalance by downweighting well-classified examples:
\begin{equation}\label{eq:focal}
    \mathcal{L}_{\text{focal}} = -\mathbb{E}_{s \sim \mathcal{D}} \left[\, \frac{1}{n} \sum_{a \in \mathcal{A}} (1 - p_a(s))^\gamma \log p_a(s) \,\right],
\end{equation}
where $p_a(s)$ is the predicted probability of the true label and $\gamma \geq 0$ is the focal parameter. For $\gamma > 0$, easy examples (where $p_a(s) \approx 1$) receive near-zero weight, concentrating gradient on hard negatives.

However, focal loss treats all actions equally. An error on an action the policy would never choose has no downstream effect, while an error on a high-probability action can severely distort the policy. This motivates our KL-balanced objective.

\textbf{KL-balanced loss} weights each action by its policy sensitivity: how much the policy would change if that action's validity were misclassified. Consider the policy using predicted versus oracle masks:
\begin{equation}
    \pi^{\text{pred}}_\theta(a \mid s) = \frac{\exp(z_a(s)) \cdot \mathbf{1}[\hat{\nu}(s, a) > \tau]}{\sum_j \exp(z_j(s)) \cdot \mathbf{1}[\hat{\nu}(s, j) > \tau]}.
\end{equation}
The KL divergence between $\pi^{\text{oracle}}$ and $\pi^{\text{pred}}$ measures the policy's sensitivity to the validity prediction of action $a$:
\begin{equation}\label{eq:kl_weight}
    w_a(s) = \pi_\theta(a \mid s) \left| \log \pi^{\text{oracle}}_\theta(a \mid s) - \log \pi^{\text{pred}}_\theta(a \mid s) \right|.
\end{equation}
To prevent infinite weights when invalid actions would have zero probability under either policy, we use a soft mask of $-20$ instead of $-\infty$ for invalid actions, yielding probabilities near zero but strictly positive. This ensures the KL divergence remains finite while preserving the weighting behavior.

The KL-balanced loss replaces uniform $1/n$ weighting with normalized KL weights:
\begin{equation}\label{eq:kl_balanced}
    \mathcal{L}_{\text{KL}} = -\mathbb{E}_{s \sim \mathcal{D}} \left[\, \sum_{a \in \mathcal{A}} \tilde{w}_a(s)\, (1 - p_a(s))^\gamma \log p_a(s) \,\right],
\end{equation}
where $\tilde{w}_a(s) = w_a(s) / \sum_{b \in \mathcal{A}} w_b(s)$. The focal parameter $\gamma$ is retained to downweight easy examples within the reweighted distribution.

The weighting mechanism is intuitive. If action $a$ carries high policy mass and the classifier predicts it invalid, the predicted policy will incorrectly zero it out, causing large KL divergence from the oracle. The KL-balanced loss assigns high classification weight to this state-action pair, inducing features that correctly distinguish validity. Conversely, actions that the policy would ignore regardless of validity contribute little to the KL and receive low weight. This focuses representation learning on validity predictions that actually affect behavior.

%% file: sections/06_experiments.tex
\section{Experiments}
\label{sec:experiments}

We structure evaluation around four questions that progressively build from identifying the problem to validating the solution:
\textbf{(i)} Does unmasked training induce exponential suppression of valid actions at unvisited states?
\textbf{(ii)} Does oracle masking increase representational entanglement, and can classification mitigate this?
\textbf{(iii)} Does classification improve performance even when oracle masks are available?
\textbf{(iv)} Can classification enable deployment without oracle masks?

\textbf{Setup.} We evaluate on Craftax and Craftax-Classic (43 actions) and MiniHack Corridor-5 (11 actions). These environments contain \emph{critical but rare} actions: \textsf{descend} (valid only at staircases) and \textsf{open\_door} (valid only adjacent to closed doors). Such actions are necessary for task completion, yet applicable at only a small fraction of states---exactly the regime where Theorem~\ref{thm:prob_suppression} predicts suppression. All environments expose ground-truth validity functions, which we use both for oracle masking and for supervision of the feasibility classifier during training.

We compare four conditions: \textbf{C1} Masked (oracle $\nu(s,a)$), \textbf{C2} Unmasked (full action space), \textbf{C3} Masked + Focal ($\lambda=10$), and \textbf{C4} Masked + KL-Balanced ($\lambda=10$). We evaluate PPO-Hybrid (S5 + Transformer-XL), PPO-RNN (GRU), and PPO-MLP architectures. MiniHack experiments additionally use RND for exploration. Evaluation uses the stochastic policy without exploration bonuses. See Appendix~\ref{app:training} for hyperparameters and Appendix~\ref{app:ablations} for sensitivity to $\lambda$ and $\gamma$.

\subsection{Exponential Suppression at Unvisited States (RQ1)}
\label{sec:rq1}

Figure~\ref{fig:rq1_suppression} examines two actions that are rarely valid and necessary for task completion: \textsf{descend} in Craftax (valid only at staircases) and \textsf{open\_door} in MiniHack Corridor-5 (valid only adjacent to closed doors). In Craftax, $\pi(\textsf{descend} \mid \text{valid})$ drops from uniform initialization ($1/43 \approx 0.023$) to below $10^{-4}$ within 50M frames---nearly two orders of magnitude, consistent with Theorem~\ref{thm:prob_suppression}'s exponential decay bound. The valid-action selection rate mirrors this collapse: unmasked agents select valid actions at fewer than 15\% of timesteps during early training. In Corridor-5, $\pi(\textsf{open\_door} \mid \text{valid})$ drops from $1/11 \approx 0.09$ to below $10^{-3}$, with similarly suppressed valid selection rates.

Oracle masking (C1, blue) consistently maintains high probability for both actions at valid states, approximately 0.5 to 1.0, with 100\% valid-action selection throughout. Masking decouples visited-state updates from unvisited-state probabilities.

\textbf{Recovery dynamics.} Suppression creates a sample efficiency bottleneck. In Craftax, $\pi(\textsf{descend})$ remains below a likelihood of $10^{-2}$ for over 100M frames, preventing the agent from reaching deeper dungeon floors regardless of value function quality. Recovery only begins when the agent finally visits staircases frequently enough to relearn the action's value. This delay between initial suppression and eventual recovery consumes the majority of training budget in sparse-reward environments. Theorem~\ref{thm:prob_suppression} characterizes the suppression phase; recovery depends on exploration, value accuracy, and valid-state density.

These results demonstrate that suppression is a systematic failure mode across architectures and environments (Table~\ref{tab:mask_removal}), not an artifact of model capacity or hyperparameter tuning. The effect appears consistently across PPO-MLP, PPO-RNN, and PPO-Hybrid architectures. For example, reaching a return of 1.0 requires 84\% of training frames for unmasked PPO-Hybrid versus only 63\% for oracle masking---a 21 percentage point gap in sample efficiency. While masking eliminates suppression at the policy level, it leaves the encoder's internal representation of validity unchanged, motivating our next question: what happens to representations?

\begin{figure}[t!]
    \centering
    \includegraphics[width=\textwidth]{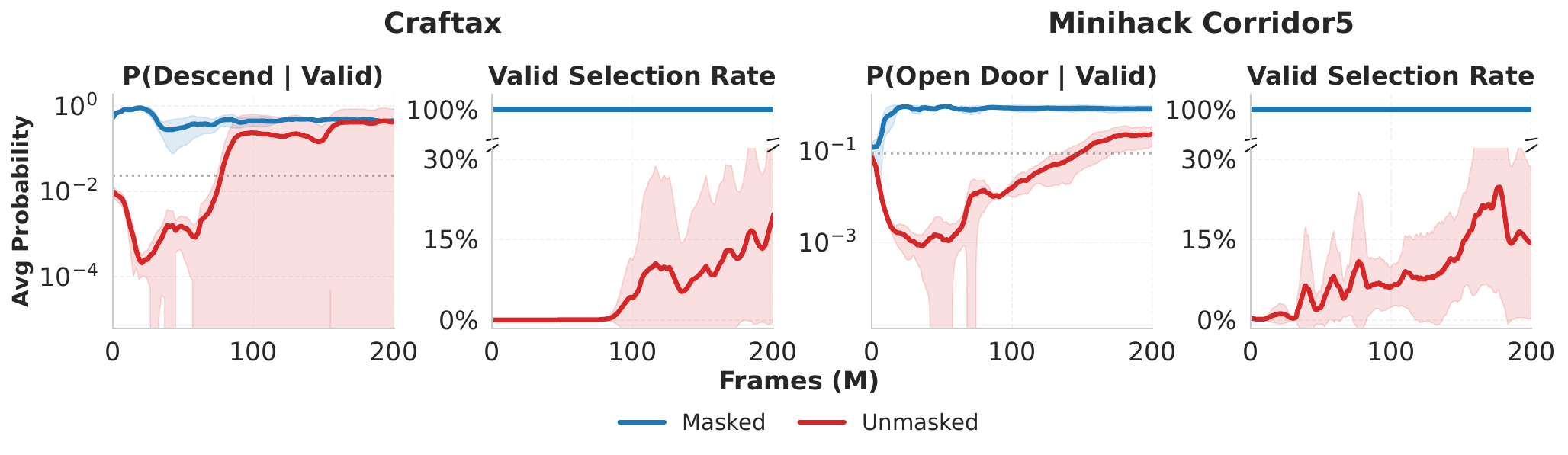}
    \caption{\textbf{Action suppression in Craftax and MiniHack Corridor-5 (PPO-Hybrid).} Left: probability of rare critical actions at valid states (log scale). Right: fraction of timesteps with valid actions selected. Unmasked training (C2, red) exhibits exponential suppression, while oracle masking (C1, blue) prevents collapse. Dashed line: uniform initialization.}
    \label{fig:rq1_suppression}
\end{figure}

\subsection{Feature Correlation Validates Condition~(ii) (RQ2)}
\label{sec:rq2}

Theorem~\ref{thm:prob_suppression} predicts suppression depends on feature correlation between valid and invalid states (Condition~(ii)). We measure this directly via Pearson correlation between encoder activations $\phi(s)$ from states where the target action is valid versus invalid.

\textbf{Oracle masking maintains entanglement.} Masked training (C1) preserves high correlation ($\approx$0.8--1.0) throughout training (Figure~\ref{fig:rq2_correlation}). Because masking resolves validity at the policy level, the encoder receives no signal to distinguish valid from invalid states. Representations stay correlated, precisely the regime where suppression would occur were masks removed.

\textbf{Unmasked training separates accidentally.} Without masking, correlation drops for individual actions; \textsf{descend} falls to approximately 0.3, for example. But this separation is a \emph{byproduct of policy collapse}, not learned validity structure. The encoder separates states because the policy has already suppressed the action, not because it learned to predict validity.

\textbf{Classification induces intentional separation.} Adding KL-balanced classification to masked training reduces correlation to $\approx$0.4 while masking prevents policy-level suppression. The classification loss provides explicit supervision for validity prediction, reshaping representations independently of the masking mechanism.

This reduction in correlation is what enables deployment without oracle masks (Section~\ref{sec:rq4}). Oracle masking alone maintains high correlation ($\approx$0.8) and fails catastrophically when masks are removed. Classification reduces correlation to $\approx$0.4, yielding learned predictors with 99\% accuracy (Appendix~\ref{app:mask_predictor_ablation}).

\begin{figure}[t!]
    \centering
    \includegraphics[width=\textwidth]{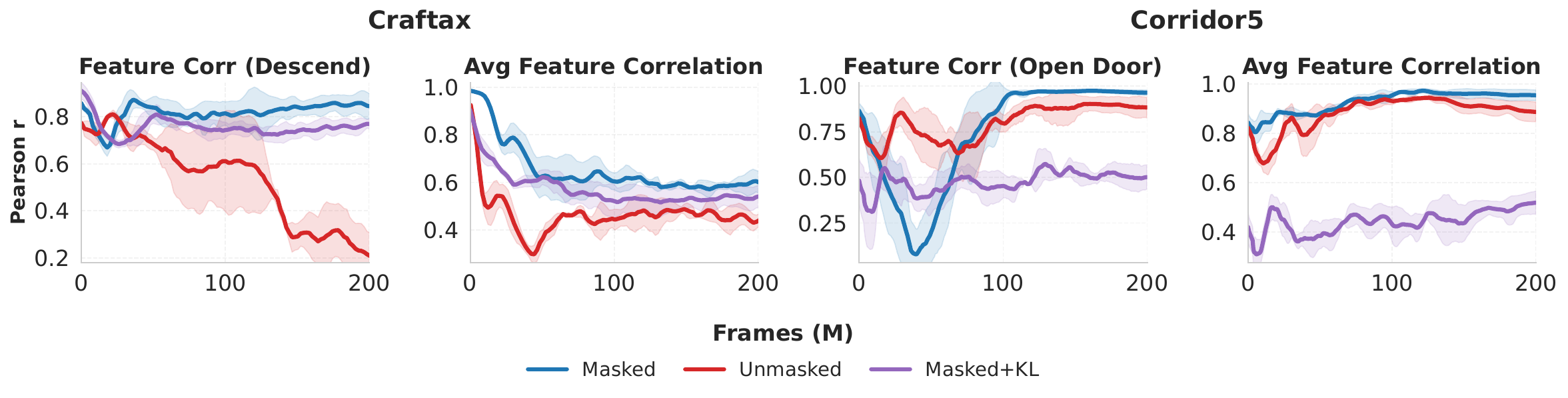}
    \caption{\textbf{Representational correlation between valid and invalid states (PPO-Hybrid).} Oracle masking preserves highly entangled representations, while KL-balanced classification induces validity-aware features without reintroducing policy-level suppression.}
    \label{fig:rq2_correlation}
\end{figure}

\subsection{Beyond Oracle Masking (RQ3)}
\label{sec:rq3}

Even with oracle masks at training time, classification improves performance. On Craftax-Hybrid, Masked + KL achieves 48.8 return versus 45.6 for masked only---a 7\% gain. This reflects the representational difference shown in Figure~\ref{fig:rq2_correlation}: masking prevents behavioral suppression but maintains correlated features, while classification induces validity-discriminating features that benefit policy learning.

The pattern holds across all tested architectures (Table~\ref{tab:mask_removal}). KL-balanced matches or exceeds oracle masking not only in final returns but in training efficiency: Masked + KL requires nearly the same fraction of frames to reach return 1.0 as oracle masking (65\% versus 63\% for PPO-Hybrid), while unmasked training lags substantially (84\%). Episode lengths reveal further differences in learned policy quality. In Craftax, oracle masking produces more efficient policies (80.7 steps per episode versus 94.6 for unmasked), suggesting that valid action selection enables more direct trajectories to rewards. In MiniHack Corridor-5, where the task is simpler and all methods achieve near-optimal returns (0.9--1.0), these differences become negligible (all around 52 steps), indicating that suppression effects are most pronounced in complex environments with sparse critical actions.

Focal loss alone exhibits higher variance and did not consistently improve over masking, highlighting the importance of policy-aware weighting. The gains from KL-balanced classification are functionally meaningful: on Craftax-Classic (max return $\approx$21), the gap between 19.8 (unmasked) and 20.5 (Masked + KL) represents the difference between occasional diamond collection ($\sim$20\% of episodes) and consistent collection ($\sim$60\%).

\subsection{Deployment Without Oracle Masks (RQ4)}
\label{sec:rq4}

When oracle masks are unavailable at test time, masked-only agents collapse to $-0.9$ return (Table~\ref{tab:mask_removal}, Pred columns). RQ2 showed why: masking prevents the encoder from learning validity-discriminating features. The policy has never encountered a validity prediction task.

KL-balanced classification solves this issue. In Craftax-Hybrid, Masked + KL achieves 43.2 return under predicted masks versus 43.9 for unmasked training with ground-truth masks. The key difference is training dynamics: unmasked training suffers suppression for 100M+ frames before recovery, while Masked + KL maintains stable valid-action probabilities throughout training and deploys immediately without oracle masks. The pattern is consistent across architectures: PPO-MLP with Masked + KL achieves 28.6 under predicted masks versus $-0.9$ for masked only, and PPO-RNN achieves 36.4 versus $-0.9$.

The MiniHack Corridor-5 results reveal a striking phenomenon: even in an environment where all methods achieve near-optimal returns under ground-truth evaluation (0.9--1.0), oracle masking fails catastrophically when masks are removed (dropping to 0.1--0.3). This reveals that masking's success hides a fundamental representational deficiency. The policy appears to work only because the oracle intervenes at every step. Without this intervention, the agent has no notion of when \textsf{open\_door} is appropriate. KL-balanced classification enables robust deployment because it learns validity structure rather than relying on external guidance.

\textbf{Practical benefit.} Predicted-mask deployment provides an explicit validity model that practitioners can inspect, audit, and selectively correct. Where the classifier errs, targeted human intervention can patch specific failure modes---an impossibility with oracle masking, which offers no learned validity representation.

\begin{table*}[t!]
\centering
\scriptsize
\setlength{\tabcolsep}{4pt}
\renewcommand{\arraystretch}{1.10}
\begin{tabular}{@{}lcccccccc@{}}
\toprule
 & \multicolumn{2}{c}{Craftax} & \multicolumn{2}{c}{Classic} & \multicolumn{4}{c}{Corridor5} \\
\cmidrule(lr){2-3}
\cmidrule(lr){4-5}
\cmidrule(lr){6-9}
 & GT & Pred & GT & Pred & GT & Pred & Frames$\to$1.0 (\%) & Ep.\ Len \\
\midrule
\multicolumn{9}{l}{\textbf{PPO-MLP}} \\
\addlinespace[1pt]
C2: Unmasked & 26.9 {\scriptsize$\pm$0.5} & 26.9 {\scriptsize$\pm$0.5} & 19.2 {\scriptsize$\pm$0.0} & 19.2 {\scriptsize$\pm$0.0} & 0.9 {\scriptsize$\pm$0.0} & 0.9 {\scriptsize$\pm$0.0} & --- & 161.3 {\scriptsize$\pm$20.2} \\
C1: Masked & \textbf{30.4 {\scriptsize$\pm$0.2}} & -0.9 {\scriptsize$\pm$0.0} & 19.3 {\scriptsize$\pm$0.0} & -0.8 {\scriptsize$\pm$0.0} & 0.9 {\scriptsize$\pm$0.0} & 0.1 {\scriptsize$\pm$0.0} & --- & \textbf{118.3 {\scriptsize$\pm$8.6}} \\
C3: Masked + Focal & 24.6 {\scriptsize$\pm$2.9} & 13.7 {\scriptsize$\pm$1.5} & 19.4 {\scriptsize$\pm$0.0} & 15.3 {\scriptsize$\pm$0.6} & \textbf{0.9 {\scriptsize$\pm$0.0}} & \textbf{0.9 {\scriptsize$\pm$0.0}} & --- & \textbf{113.6 {\scriptsize$\pm$8.1}} \\
C4: Masked + KL & \textbf{28.7 {\scriptsize$\pm$3.2}} & \textbf{28.6 {\scriptsize$\pm$3.2}} & \textbf{19.5 {\scriptsize$\pm$0.0}} & \textbf{19.3 {\scriptsize$\pm$0.1}} & \textbf{0.9 {\scriptsize$\pm$0.0}} & \textbf{0.9 {\scriptsize$\pm$0.0}} & --- & \textbf{123.6 {\scriptsize$\pm$6.6}} \\
\addlinespace[2pt]
\midrule
\multicolumn{9}{l}{\textbf{PPO-RNN}} \\
\addlinespace[1pt]
C2: Unmasked & 36.7 {\scriptsize$\pm$0.1} & 36.7 {\scriptsize$\pm$0.1} & 19.8 {\scriptsize$\pm$0.0} & 19.8 {\scriptsize$\pm$0.0} & 1.0 {\scriptsize$\pm$0.0} & 1.0 {\scriptsize$\pm$0.0} & 29 {\scriptsize$\pm$2} & 52.0 {\scriptsize$\pm$0.9} \\
C1: Masked & 38.3 {\scriptsize$\pm$0.2} & -0.9 {\scriptsize$\pm$0.0} & 20.4 {\scriptsize$\pm$0.0} & -0.9 {\scriptsize$\pm$0.0} & \textbf{1.0 {\scriptsize$\pm$0.0}} & 0.2 {\scriptsize$\pm$0.1} & \textbf{16 {\scriptsize$\pm$2}} & \textbf{52.1 {\scriptsize$\pm$1.1}} \\
C3: Masked + Focal & \textbf{39.2 {\scriptsize$\pm$0.1}} & 27.5 {\scriptsize$\pm$1.2} & 20.3 {\scriptsize$\pm$0.0} & 14.8 {\scriptsize$\pm$0.2} & \textbf{1.0 {\scriptsize$\pm$0.0}} & \textbf{1.0 {\scriptsize$\pm$0.0}} & 21 {\scriptsize$\pm$2} & \textbf{52.0 {\scriptsize$\pm$0.7}} \\
C4: Masked + KL & 38.1 {\scriptsize$\pm$0.1} & \textbf{36.4 {\scriptsize$\pm$0.2}} & \textbf{20.5 {\scriptsize$\pm$0.1}} & \textbf{20.4 {\scriptsize$\pm$0.1}} & \textbf{1.0 {\scriptsize$\pm$0.0}} & \textbf{1.0 {\scriptsize$\pm$0.0}} & \textbf{21 {\scriptsize$\pm$3}} & \textbf{53.8 {\scriptsize$\pm$2.7}} \\
\addlinespace[2pt]
\midrule
\multicolumn{9}{l}{\textbf{PPO-Hybrid}} \\
\addlinespace[1pt]
C2: Unmasked & 43.9 {\scriptsize$\pm$0.4} & 43.9 {\scriptsize$\pm$0.4} & 19.2 {\scriptsize$\pm$0.1} & 19.2 {\scriptsize$\pm$0.1} & 1.0 {\scriptsize$\pm$0.0} & 1.0 {\scriptsize$\pm$0.0} & 84 {\scriptsize$\pm$11} & 94.6 {\scriptsize$\pm$7.9} \\
C1: Masked & 45.6 {\scriptsize$\pm$1.3} & -0.9 {\scriptsize$\pm$0.0} & 18.9 {\scriptsize$\pm$0.0} & -0.9 {\scriptsize$\pm$0.0} & \textbf{1.0 {\scriptsize$\pm$0.0}} & 0.3 {\scriptsize$\pm$0.0} & \textbf{63 {\scriptsize$\pm$12}} & \textbf{80.7 {\scriptsize$\pm$4.3}} \\
C3: Masked + Focal & 45.0 {\scriptsize$\pm$2.3} & 21.2 {\scriptsize$\pm$1.7} & \textbf{19.0 {\scriptsize$\pm$0.1}} & 17.6 {\scriptsize$\pm$0.5} & \textbf{1.0 {\scriptsize$\pm$0.0}} & \textbf{1.0 {\scriptsize$\pm$0.0}} & \textbf{76 {\scriptsize$\pm$8}} & \textbf{92.0 {\scriptsize$\pm$10.9}} \\
C4: Masked + KL & \textbf{48.8 {\scriptsize$\pm$0.7}} & \textbf{43.2 {\scriptsize$\pm$1.1}} & \textbf{19.0 {\scriptsize$\pm$0.1}} & \textbf{18.4 {\scriptsize$\pm$0.1}} & \textbf{1.0 {\scriptsize$\pm$0.0}} & \textbf{1.0 {\scriptsize$\pm$0.0}} & \textbf{65 {\scriptsize$\pm$5}} & \textbf{88.0 {\scriptsize$\pm$7.8}} \\
\bottomrule
\end{tabular}
\caption{\textbf{Evaluation with ground-truth (GT) versus predicted (Pred) masks.} Masked-only agents fail when oracle masks are removed. KL-balanced classification enables near-oracle performance under predicted masks across architectures. Results report mean $\pm$ std over 4 seeds (10,240 episodes per seed); bold indicates best non-baseline per column.}
\label{tab:mask_removal}
\vspace{-1em}
\end{table*}

%% file: sections/07_conclusion.tex
\section{Conclusion and Limitations}
\label{sec:conclusion}

Action masking prevents valid action suppression, but it introduces a practical deployment dilemma. Practitioners must either maintain costly validity oracles at deployment or accept policy degradation. We proved that suppression grows exponentially through shared parameters, with rarely-valid actions affected most severely. Feasibility classification addresses this failure by restoring a gradient signal that enables the learning of validity-discriminating features. This yields a practical training-deployment strategy: train with oracle masking for stability, then substitute the learned predictor at test time. Our KL-balanced loss outperforms uniform weighting and enables deployment without oracle masks at minimal performance cost. In practice, classification accuracy on rarely-valid actions should be prioritized, as these actions are most vulnerable to suppression.

Several limitations remain. Our theoretical analysis assumes fixed state features $\phi(s)$. Although Section~\ref{sec:rq2} empirically validates that deep networks exhibit the feature alignment Condition~(ii), we do not analyze how joint optimization of the representation $\phi$ and the policy weights alters suppression dynamics. Deployment with predicted masks recovers unmasked performance but does not exceed it, as the policy is trained under oracle masking and experiences a distribution shift when evaluated on predicted masks. Nonetheless, matching unmasked performance while providing an explicit validity model offers meaningful interpretability benefits. Practitioners can inspect disagreements between the learned predictor and the oracle, identifying potential failures prior to deployment. They can also intervene to correct mispredictions during deployment. This mirrors concept-bottleneck interventions in multi-agent settings \citep{zabounidis2023concept}.

One natural extension is to train a separate policy head via behavioral cloning that relies exclusively on predicted masks throughout training, thereby reducing the train-deploy distribution shift. When validity oracles are inexpensive at test time, the primary benefit of feasibility classification is representational. KL-balanced classification improves performance even when oracle masks are available (Section~\ref{sec:rq3}). The deployment advantage is most pronounced in sim-to-real transfer, where validity information is readily available in simulation but costly or unavailable on physical hardware.

Our analysis assumes access to ground-truth validity labels during training. Noisy or partially specified validity functions would likely degrade the learned predictor. Our evaluation is limited to discrete action spaces with symbolic state representations. Continuous action spaces or pixel-based observations may exhibit different feature-alignment dynamics. The suppression mechanism may also interact with other known failure modes in deep reinforcement learning, including representation collapse, loss of plasticity, and primacy bias.

%% file: appendix_sections/appendix_proofs.tex
\section{Supporting Lemmas and Proofs}
\label{app:proofs}

\subsection{Logit Gradient Identity}

\begin{lemma}[Expected logit update]\label{lem:logit_update}
Under softmax parameterization, the policy gradient induces the following expected update to the logit of action $a$ at state $s$:
\begin{equation}\label{eq:logit_update}
    \mathbb{E}_{a' \sim \pi(\cdot \mid s)}\!\left[ A^\pi(s, a') \cdot \frac{\partial \log \pi(a' \mid s)}{\partial z_a(s)} \right] = \pi(a \mid s) \cdot A^\pi(s, a).
\end{equation}
Marginalizing over the state visitation gives the full update direction:
\begin{equation}\label{eq:logit_update_full}
    \mathbb{E}_{s \sim d_\pi}\!\left[\, \pi(a \mid s)\, A^\pi(s, a) \,\right].
\end{equation}
\end{lemma}

\begin{proof}
The derivative of the softmax log-probability is
\begin{equation}
    \frac{\partial \log \pi(a' \mid s)}{\partial z_a(s)} = \mathbf{1}\{a' = a\} - \pi(a \mid s).
\end{equation}
Taking the expectation over $a' \sim \pi(\cdot \mid s)$:
\begin{align}
    \mathbb{E}_{a'}\!\left[ A^\pi(s, a') \left(\mathbf{1}\{a' = a\} - \pi(a \mid s)\right) \right]
    &= \pi(a \mid s) \cdot A^\pi(s, a) - \pi(a \mid s) \sum_{a' \in \mathcal{A}} \pi(a' \mid s)\, A^\pi(s, a') \notag \\
    &= \pi(a \mid s) \cdot A^\pi(s, a)
\end{align}
since $\sum_{a'} \pi(a' \mid s)\, A^\pi(s, a') = 0$. Equation~\eqref{eq:logit_update_full} follows by taking the outer expectation over $s \sim d_\pi$.
\end{proof}

\subsection{Zero-Sum Logit Updates}

\begin{lemma}[Zero-sum identity]\label{lem:zero_sum}
Under the linear parameterization $z_j(s) = \phi(s)^\top w_j$ with per-action parameters $w_j$, the sum of logit updates across all actions is zero at every state and every step:
\begin{equation}
    \sum_{j=1}^n \Delta z_j^{(\tau)}(s) = 0 \quad \text{for all } s, \tau.
\end{equation}
Consequently, $\sum_j e^{\sigma_j^{(T)}} \geq n$ for all $T$, where $\sigma_j^{(T)} = \sum_{\tau=0}^{T-1} \Delta z_j^{(\tau)}(s)$ is the cumulative logit change from uniform initialization.
\end{lemma}

\begin{proof}
By Lemma~\ref{lem:logit_update}, the expected parameter update is $\Delta w_j = \eta\, \mathbb{E}_{s \sim d_\pi}[\pi(j \mid s)\, A^\pi(s, j)\, \phi(s)]$. Summing over actions:
\begin{equation}
    \sum_j \Delta w_j = \eta\, \mathbb{E}_{s \sim d_\pi}\!\left[\sum_j \pi(j \mid s)\, A^\pi(s, j)\, \phi(s)\right] = 0
\end{equation}
since $\sum_j \pi(j \mid s)\, A^\pi(s, j) = 0$ for every $s$. Therefore $\sum_j \Delta z_j^{(\tau)}(s) = \phi(s)^\top \sum_j \Delta w_j^{(\tau)} = 0$.

The same identity holds for stochastic gradient estimates. For a single sample $(s, a')$, the per-sample update direction for $w_j$ is proportional to $(\mathbf{1}\{a' = j\} - \pi(j \mid s))\, \phi(s)$. Summing: $\sum_j (\mathbf{1}\{a' = j\} - \pi(j \mid s)) = 1 - 1 = 0$.

For the partition function bound: since $\sum_j \sigma_j^{(T)} = 0$, Jensen's inequality convexity of the exponential with the uniform distribution over $n$ actions gives
\begin{equation}
    \frac{1}{n}\sum_j e^{\sigma_j^{(T)}} \geq e^{\frac{1}{n}\sum_j \sigma_j^{(T)}} = e^0 = 1,
\end{equation}
Thus $\sum_j e^{\sigma_j^{(T)}} \geq n$.
\end{proof}

\subsection{Invalid-Action Dominance Gap}

\begin{lemma}[Invalid-action dominance gap]\label{lem:dominance_gap}
Suppose that for all states $s \in \mathcal{S}_{\mathrm{vis}}$ and all invalid actions $a$ where $\nu(s, a) = 0$,
\begin{equation}
    Q^\pi(s, a) \leq \max_{b : \nu(s, b) = 1} Q^\pi(s, b) - \delta(s)
\end{equation}
for some $\delta(s) > 0$. Then
\begin{equation}\label{eq:advantage_bound}
    A^\pi(s, a) \leq -\delta(s)\,\bigl(1 - \pi(a \mid s)\bigr).
\end{equation}
\end{lemma}

\begin{proof}
Let $b^* = \arg\max_{b : \nu(s, b) = 1} Q^\pi(s, b)$. Since $Q^\pi(s, b^*) \geq Q^\pi(s, a')$ for all $a'$,
\begin{align}
    A^\pi(s, a) &= Q^\pi(s, a) - \sum_{a'} \pi(a' \mid s)\, Q^\pi(s, a') \notag \\
    &\leq Q^\pi(s, a) - \pi(a \mid s)\, Q^\pi(s, a) - (1 - \pi(a \mid s))\, Q^\pi(s, b^*) \notag \\
    &= (1 - \pi(a \mid s))\bigl(Q^\pi(s, a) - Q^\pi(s, b^*)\bigr) \leq -(1 - \pi(a \mid s))\,\delta(s).\qedhere
\end{align}
\end{proof}

\subsection{Parameter Updates Under Linear Logits}

\begin{lemma}[Expected parameter update]\label{lem:param_update}
Under the linear parameterization $z_a(s) = \phi(s)^\top w_a$ with gradient ascent at learning rate $\eta$, the expected update to $w_a$ is
\begin{equation}\label{eq:param_update}
    \Delta w_a = \eta \, \mathbb{E}_{s \sim d_\pi}\!\left[\, \pi(a \mid s)\, A^\pi(s, a)\, \phi(s) \,\right].
\end{equation}
If action $a$ is invalid at every state in $\mathcal{S}_{\mathrm{vis}}$ and condition in Equation~\eqref{eq:dominance} holds, then
\begin{equation}\label{eq:param_update_neg}
    \Delta w_a = -\eta \, \mathbb{E}_{s \sim d_\pi}\!\left[\, c(s)\, \phi(s) \,\right]
\end{equation}
where $c(s) = \pi(a \mid s)\, |A^\pi(s, a)| > 0$ for all $s \in \mathcal{S}_{\mathrm{vis}}$.
\end{lemma}

\begin{proof}
The gradient of $J$ with respect to $w_a$ is
\begin{equation}
    \nabla_{w_a} J = \mathbb{E}_{s \sim d_\pi,\, a' \sim \pi}\!\left[ A^\pi(s, a')\, \frac{\partial \log \pi(a' \mid s)}{\partial z_a(s)}\, \phi(s) \right].
\end{equation}
By Lemma~\ref{lem:logit_update}, the inner expectation over $a'$ at each $s$ yields $\pi(a \mid s)\, A^\pi(s, a)$, giving Equation~\eqref{eq:param_update}. When $a$ is invalid at all states in $\mathcal{S}_{\mathrm{vis}}$, Lemma~\ref{lem:dominance_gap} gives $A^\pi(s, a) < 0$, so $\pi(a \mid s)\, A^\pi(s, a) = -c(s)$ with $c(s) > 0$.
\end{proof}

\subsection{Generalization of Logit Suppression}

\begin{proposition}[Generalization of logit suppression]\label{prop:suppression}
Consider a policy with linear logits $z_a(s) = \phi(s)^\top w_a$. Suppose action $a$ is invalid at every state in $\mathcal{S}_{\mathrm{vis}}$ and condition~\eqref{eq:dominance} holds on $\mathcal{S}_{\mathrm{vis}}$.

Then for any $s' \in \mathcal{S}_{\mathrm{unvis}}$, the logit change after one gradient step is
\begin{equation}\label{eq:suppression}
    \Delta z_a(s') = -\eta\, \phi(s')^\top \mathbb{E}_{s \sim d_\pi}\!\left[\, c(s)\, \phi(s)\,\right]
\end{equation}
which is strictly negative whenever
\begin{equation}
    \phi(s')^\top \mathbb{E}_{s \sim d_\pi}\!\left[\, c(s)\, \phi(s)\,\right] > 0,
\end{equation}
that is, whenever $\phi(s')$ is not orthogonal to the advantage-weighted mean of visited state features.
\end{proposition}

\begin{proof}
From Lemma~\ref{lem:param_update}, $\Delta w_a = -\eta\, \mathbb{E}_{s \sim d_\pi}[c(s)\, \phi(s)]$. The logit change at $s'$ is
\begin{equation}
    \Delta z_a(s') = \phi(s')^\top \Delta w_a = -\eta\, \phi(s')^\top \mathbb{E}_{s \sim d_\pi}\!\left[c(s)\, \phi(s)\right].
\end{equation}
Since $c(s) > 0$ for all $s \in \mathcal{S}_{\mathrm{vis}}$, the vector $\mathbb{E}_{s \sim d_\pi}[c(s)\, \phi(s)]$ is a positive combination of visited feature vectors, and the result follows.
\end{proof}

\subsection{Proof of Theorem~\ref{thm:prob_suppression}}

\begin{proof}
We first derive the per-step suppression rate $\kappa_\tau$ and then bound the probability.

\paragraph{Per-step suppression.}
By Proposition~\ref{prop:suppression}, the logit change of action $a$ at the unvisited state $s^*$ after one gradient step at time $\tau$ is
\begin{equation}
    \Delta z_a^{(\tau)}(s^*) = -\eta\, \phi(s^*)^\top \mathbb{E}_{s \sim d_{\pi_\tau}}\!\left[\, c_\tau(s)\, \phi(s)\,\right]
\end{equation}
where $c_\tau(s) = \pi_\tau(a \mid s)\, |A^{\pi_\tau}(s, a)| > 0$ for all $s \in \mathcal{S}_{\mathrm{vis}}$. Since $a$ is invalid at every visited state, Lemma~\ref{lem:dominance_gap} gives $A^{\pi_\tau}(s, a) < 0$. The per-step suppression rate defined in the theorem,
\begin{equation}
    \kappa_\tau = \eta\, \phi(s^*)^\top \mathbb{E}_{s \sim d_{\pi_\tau}}\!\left[\, c_\tau(s)\, \phi(s)\,\right],
\end{equation}
equals the magnitude of this logit decrease: $\Delta z_a^{(\tau)}(s^*) = -\kappa_\tau$. Condition~ii ensures $\kappa_\tau > 0$ at every step.

\paragraph{Partition function bound.}
Write $\sigma_j^{(T)} = \sum_{\tau=0}^{T-1} \Delta z_j^{(\tau)}(s^*)$ for the cumulative logit change of action $j$ at $s^*$. By Lemma~\ref{lem:zero_sum} on the zero-sum identity, $\sum_j \sigma_j^{(T)} = 0$. Jensen's inequality (convexity of $\exp$, uniform distribution over $n$ actions) then gives
\begin{equation}
    \sum_j e^{\sigma_j^{(T)}} \geq n\, e^{\frac{1}{n}\sum_j \sigma_j^{(T)}} = n.
\end{equation}

\paragraph{Probability bound.}
Since the common initial logit $c$ cancels in $\pi_T(a \mid s^*) = e^{\sigma_a^{(T)}} / \sum_j e^{\sigma_j^{(T)}}$:
\begin{equation}
    \pi_T(a \mid s^*) \leq \frac{e^{\sigma_a^{(T)}}}{n}.
\end{equation}
From the per-step suppression, $\sigma_a^{(T)} = \sum_{\tau=0}^{T-1} \Delta z_a^{(\tau)}(s^*) = -\sum_{\tau=0}^{T-1} \kappa_\tau = -K_T$, giving $\pi_T(a \mid s^*) \leq e^{-K_T}/n$.
\end{proof}

\subsection{Entropy Regularization}
\label{app:entropy_proof}

\begin{proof}[Proof that the zero-sum identity and suppression bound hold under entropy regularization]
With entropy coefficient $\beta > 0$, the per-step objective becomes $J_\beta = J + \beta\, \mathbb{E}_{s \sim d_\pi}[H(\pi(\cdot \mid s))]$.

\paragraph{Zero-sum identity under entropy regularization.}
By Lemma~\ref{lem:logit_update}, the entropy-regularized expected logit update for action $a$ at state $s$ is
\begin{equation}\label{eq:entropy_logit_update}
    \Delta z_a^{\mathrm{ent}}(s) \;=\; \pi(a \mid s)\,\bigl[\,A^\pi(s, a) \;-\; \beta \log \pi(a \mid s) \;-\; \beta\, H(\pi(\cdot \mid s))\,\bigr].
\end{equation}
Summing over all actions $a = 1, \ldots, n$:
\begin{align}
    \sum_{a=1}^n \Delta z_a^{\mathrm{ent}}(s)
    &= \sum_a \pi(a \mid s)\, A^\pi(s, a) \;-\; \beta \sum_a \pi(a \mid s) \log \pi(a \mid s) \;-\; \beta\, H(\pi(\cdot \mid s)) \sum_a \pi(a \mid s) \notag \\
    &= 0 \;+\; \beta\, H(\pi(\cdot \mid s)) \;-\; \beta\, H(\pi(\cdot \mid s)) \;=\; 0,
\end{align}
where we used $\sum_a \pi(a \mid s)\, A^\pi(s, a) = 0$ and $H(\pi(\cdot \mid s)) = -\sum_a \pi(a \mid s) \log \pi(a \mid s)$. The zero-sum identity therefore holds at every state and every step, exactly as in the unregularized case. Under the linear parameterization $z_j(s) = \phi(s)^\top w_j$, this implies $\sum_j \Delta z_j^{(\tau)}(s') = 0$ for all $s'$, so Jensen's inequality continues to give $\sum_j e^{\sigma_j^{(T)}} \geq n$ and Equation~\eqref{eq:T_step_decay} holds.

\paragraph{Vanishing correction at uniform initialization.}
At initialization, $\pi_0(a \mid s) = 1/n$ for all $a, s$. The entropy correction terms in Equation~\eqref{eq:entropy_logit_update} evaluate to
\begin{equation}
    -\beta \log \pi_0(a \mid s) - \beta\, H(\pi_0(\cdot \mid s)) = -\beta \log(1/n) - \beta \log n = \beta \log n - \beta \log n = 0.
\end{equation}
Therefore $\Delta z_a^{\mathrm{ent}}(s) = \pi_0(a \mid s)\, A^\pi(s, a)$ at the first step, identical to the unregularized update.

\paragraph{Convergence and sandwich bound.}
Recall from the proof of Theorem~\ref{thm:prob_suppression} that the unregularized per-step suppression rate is $\kappa_\tau = \eta\, \phi(s^*)^\top \mathbb{E}_{s \sim d_{\pi_\tau}}[c_\tau(s)\, \phi(s)]$, where $c_\tau(s) = \pi_\tau(a \mid s)\, |A^{\pi_\tau}(s, a)|$. Under entropy regularization, the logit update~\eqref{eq:entropy_logit_update} adds correction terms that modify the suppression rate. Using the same propagation via parameter sharing as in Proposition~\ref{prop:suppression}, the entropy-regularized suppression at step $\tau$ becomes
\begin{equation}
    \kappa_\tau^{\mathrm{ent}} = \kappa_\tau - \beta\, \phi(s^*)^\top \mathbb{E}_{s \sim d_{\pi_\tau}}\!\left[\, \pi_\tau(a \mid s)\, \bigl(\log \pi_\tau(a \mid s) + H(\pi_\tau(\cdot \mid s))\bigr)\, \phi(s) \,\right].
\end{equation}

As $\pi_\tau(a \mid s) \to 0$ at visited states where $a$ is invalid, the correction term $\pi_\tau(a \mid s) \log \pi_\tau(a \mid s) \to 0$, so the entropy correction at visited states vanishes. The key effect operates at $s^*$ itself: as $\pi_\tau(a \mid s^*)$ decreases, the entropy bonus $-\beta \log \pi_\tau(a \mid s^*)$ grows without bound, while the advantage $|A^\pi(s, a)| \leq r_{\max}/(1 - \gamma)$ remains bounded for all $s$ and $\pi$.

\emph{Lower bound.} The entropy-regularized logit update for action $a$ at any state includes the term $-\beta \log \pi(a \mid s)$, which exceeds the maximum advantage magnitude $r_{\max}/(1-\gamma)$ whenever $\pi(a \mid s) < \exp(-r_{\max}/(\beta(1-\gamma)))$. Below this threshold the entropy bonus dominates, preventing further suppression. Therefore
\begin{equation}
    \pi_\infty(a \mid s^*) \;\geq\; \exp\!\left(-\frac{r_{\max}}{\beta(1 - \gamma)}\right).
\end{equation}

\emph{Upper bound and convergence.} Since suppression halts at a finite logit displacement, the cumulative suppression converges: $K_\infty = \sum_{\tau=0}^\infty \kappa_\tau^{\mathrm{ent}} < \infty$. Applying the same Jensen bound as in the proof of Equation~\eqref{eq:T_step_decay}:
\begin{equation}
    \pi_\infty(a \mid s^*) \;\leq\; \frac{e^{-K_\infty}}{n} \;>\; 0.
\end{equation}
Combining yields the sandwich~\eqref{eq:entropy_sandwich}.\qedhere
\end{proof}

%% file: appendix_sections/appendix_environments.tex
\section{Environment Details}
\label{app:environments}

\begin{figure}[t]
  \centering
  \begin{subfigure}[t]{0.48\textwidth}
    \centering
    \includegraphics[width=\textwidth]{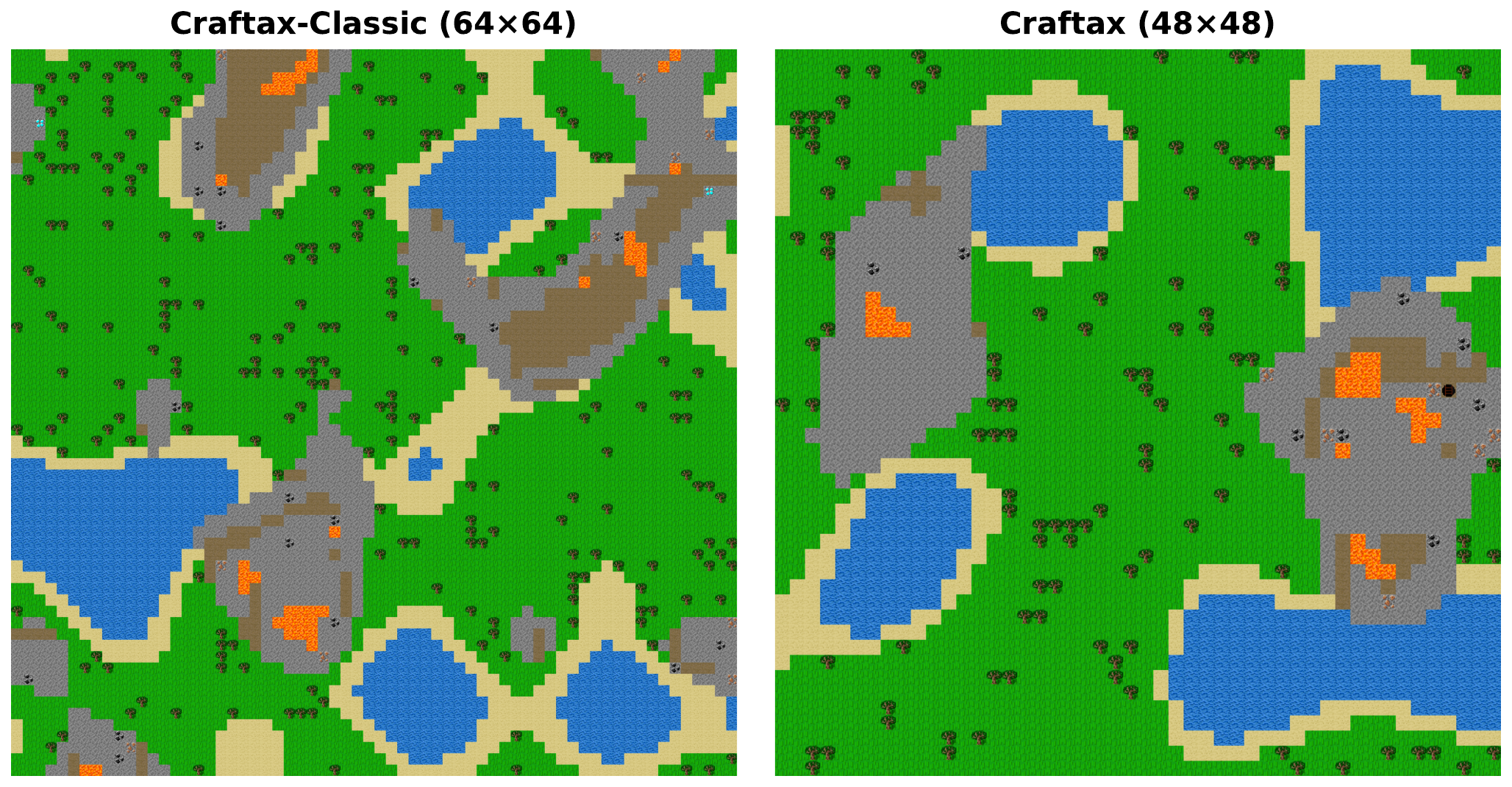}
    \caption{Craftax Classic (left) and Craftax (right).}
    \label{fig:env_craftax}
  \end{subfigure}
  \hfill
  \begin{subfigure}[t]{0.48\textwidth}
    \centering
    \includegraphics[width=\textwidth]{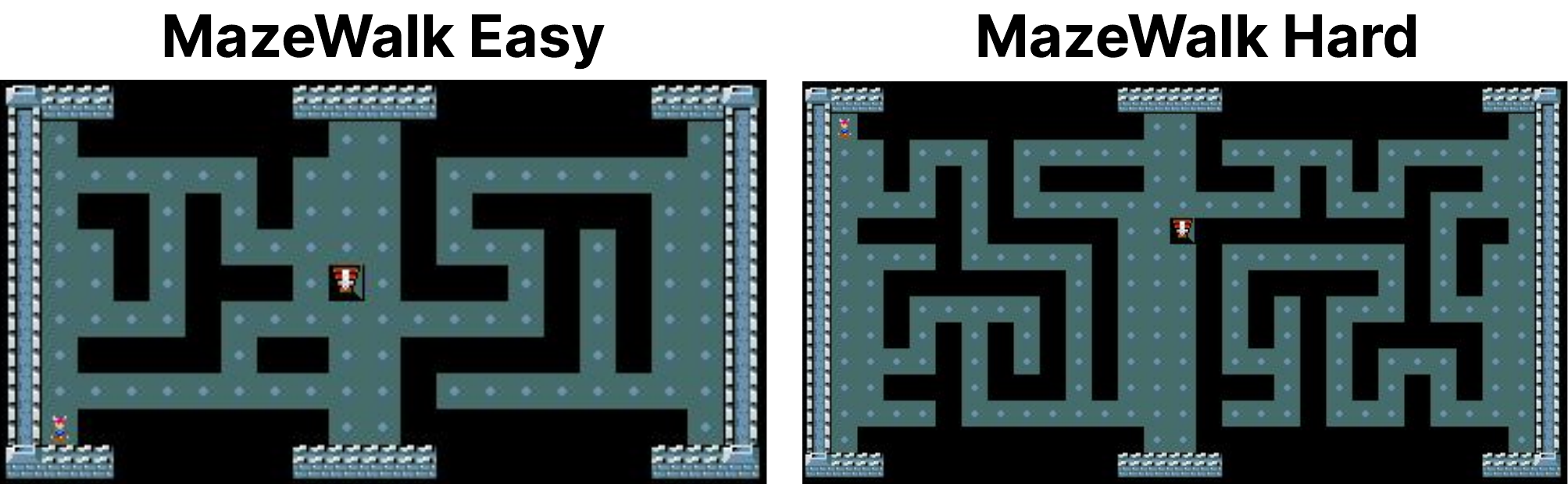}
    \caption{MazeWalk Easy (left) and MazeWalk Hard (right).}
    \label{fig:env_maze}
  \end{subfigure}

  \vspace{0.5em}

  \begin{subfigure}[t]{\textwidth}
    \centering
    \includegraphics[width=\textwidth]{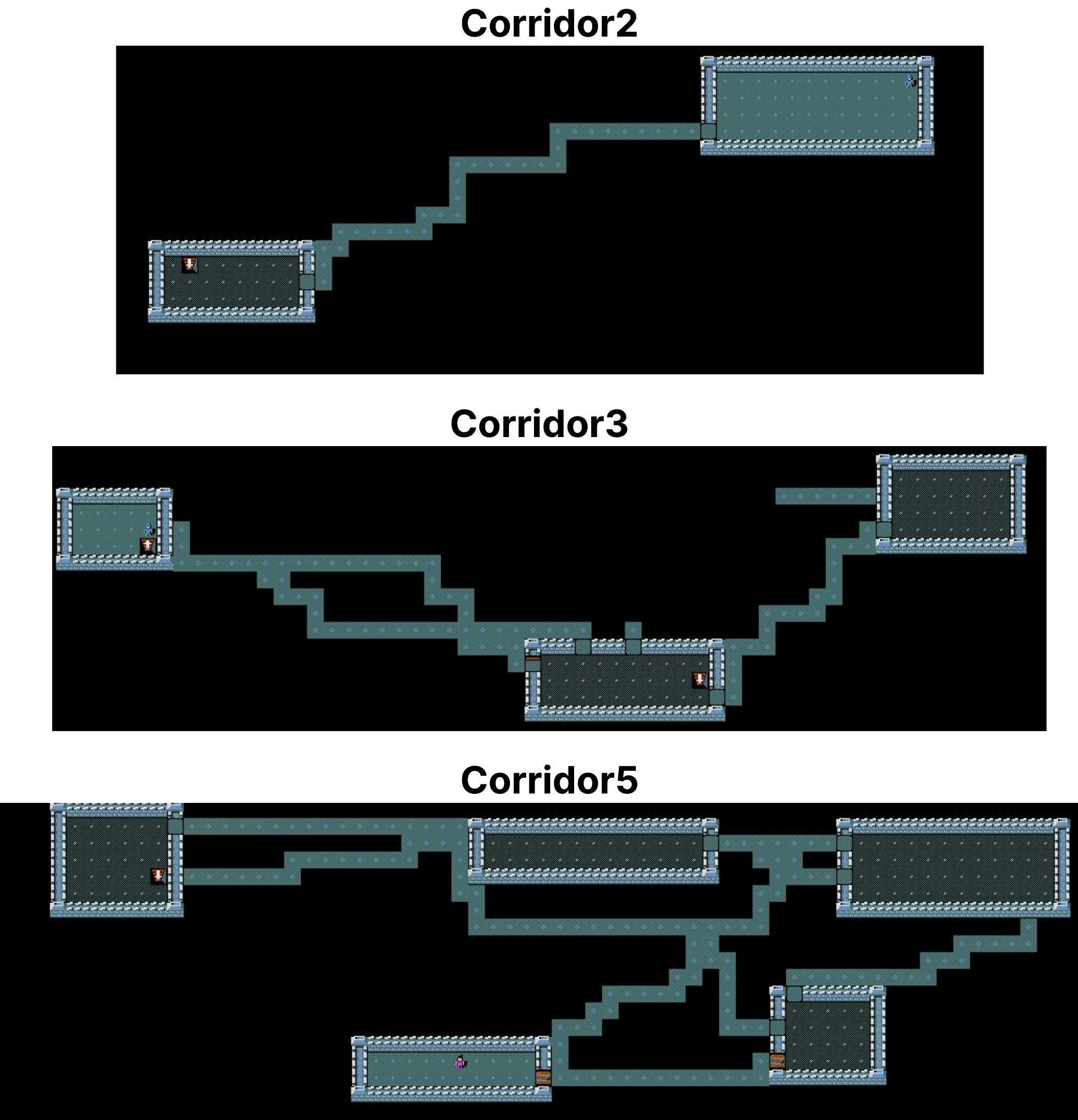}
    \caption{MiniHack Corridors with 2, 3, and 5 rooms.}
    \label{fig:env_corridor}
  \end{subfigure}
  \caption{Environments used in our experiments.}
  \label{fig:environments}
\end{figure}

\subsection{Craftax}

Craftax~\citep{craftax} is a JAX-based survival environment. We evaluate on two variants: Craftax-Classic (17 actions, single floor) and Craftax-Full (43 actions, 9 floors).

\paragraph{Observation space.}
Observations are symbolic vectors constructed by the environment renderer:

\textbf{Craftax-Classic: 1345 dimensions.} The 7$\times$9 local view (21 channels per tile) is flattened in C-order (row-major) to 1323 dimensions. Channels include: 17-channel one-hot block type, 4 binary mob channels. Inventory (12 dims): wood, stone, coal, iron, diamond, sapling, and tool counts, each divided by 10. Intrinsics (4 dims): health, food, drink, energy, each divided by 10. Direction (4 dims): one-hot encoding. Special (2 dims): light level and is\_sleeping flag.

\textbf{Craftax-Full: 8268 dimensions.} The 9$\times$11 local view (83 channels per tile) is flattened to 8217 dimensions. Channels include: 37-channel one-hot block type, 5-channel one-hot item type, 40 mob channels, and 1 light map channel. Darkness masking zeros all channels where light $\leq$ 0.05. Inventory (16 dims), potions (6 dims), intrinsics (9 dims), direction (4 dims), armour (8 dims), and special flags (8 dims) follow the normalization scheme in Table~\ref{tab:craftax_obs}.

\begin{table}[h]
\centering
\caption{Craftax-Full observation components.}
\label{tab:craftax_obs}
\small
\begin{tabular}{lll}
\toprule
\textbf{Component} & \textbf{Dims} & \textbf{Normalization} \\
\midrule
Map & 8217 & Flattened 9$\times$11$\times$83, darkness masked \\
Inventory & 16 & $\sqrt{\text{count}}/10$ for materials; level/4 for tools \\
Potions & 6 & $\sqrt{\text{count}}/10$ \\
Intrinsics & 9 & value/10 \\
Direction & 4 & One-hot \\
Armour & 8 & level/2 for pieces; raw int for enchantments \\
Special & 8 & Raw values or scaled \\
\bottomrule
\end{tabular}
\end{table}

\paragraph{Action space.}
Craftax-Classic has 17 actions: movement (4), interaction (2), placement (5), and crafting up to iron-tier (6). Craftax-Full has 43 actions including additional crafting (diamond tools, armour), combat (shoot arrow, cast spells), consumables (6 potions), progression (descend, ascend, rest), and character development (read book, enchant, level up attributes).

\paragraph{Validity function.}
The ground-truth validity function is self-implemented in \texttt{praxis/action\_masking/craftax/masks.py}. It extracts symbolic state from the environment (inventory counts, nearby blocks, player status) and applies hand-coded precondition rules that mirror the game logic in \texttt{craftax/game\_logic.py}. For example, \textsf{descend} requires standing on a down-ladder and having killed $\geq$ 8 monsters on the current floor; crafting requires nearby crafting tables/furnaces and sufficient materials.

Invalid actions result in silent no-ops with zero reward (no penalty). The reward function computes achievement deltas; invalid actions produce no achievements and no health change, yielding reward = 0.

\paragraph{Procedural generation.}
Generation is deterministic given a JAX PRNG key. Each floor is 48$\times$48, generated via fractal Perlin noise (1 octave, persistence 0.5, lacunarity 2). Terrain thresholds vary by level. Dungeons use 8 rooms (5--10 tiles each dimension) connected by L-shaped corridors. Ore spawn probabilities per stone block:

\begin{table}[h]
\centering
\caption{Ore spawn probabilities by floor.}
\label{tab:ore_spawns}
\small
\begin{tabular}{lccccc}
\toprule
\textbf{Ore} & \textbf{Overworld} & \textbf{Gnomish} & \textbf{Troll} & \textbf{Fire} & \textbf{Ice} \\
\midrule
Coal & 0.03 & 0.04 & 0.04 & 0.05 & 0.0 \\
Iron & 0.02 & 0.02 & 0.03 & 0.0 & 0.0 \\
Diamond & 0.001 & 0.005 & 0.01 & 0.0 & 0.005 \\
Sapphire & 0.0 & 0.0025 & 0.01 & 0.0 & 0.02 \\
Ruby & 0.0 & 0.0025 & 0.01 & 0.025 & 0.0 \\
\bottomrule
\end{tabular}
\end{table}

Mobs spawn probabilistically each timestep (base chances: 0.10 passive, 0.02--0.06 melee, 0.05 ranged). Maximum 3 melee, 3 passive, 2 ranged per level. The down-ladder is blocked until 8 monsters are killed on that floor.

\paragraph{Reward function.}
Rewards are immediate (per-step), not at episode end. Craftax-Classic has 22 achievements worth 1.0 each (max 22). Craftax-Full has 67 achievements with tiered rewards (1/3/5/8) for a maximum of 226. Health changes contribute $\pm$0.1 per HP.

\paragraph{Episode termination.}
Craftax-Classic: 10,000 steps, health $\leq$ 0, or lava contact. Craftax-Full: 100,000 steps, health $\leq$ 0, or boss defeated (8 hits on necromancer).

\subsection{MiniHack Corridors}

The MiniHack Corridor environments~\citep{minihack2021theplanet} are procedurally generated navigation tasks. We use a JAX reimplementation (Nethax) that faithfully reproduces NetHack 3.7 mechanics.

\paragraph{Observation space.}
Observations use SymbolicGlyphNet, adapted from the Chaotic Dwarf NLE baseline. The encoder processes a 9$\times$9 crop of a compact 50-ID glyph space (tile types, monsters, items) through embedding (32 dims) and two conv layers (64 and 128 channels, 3$\times$3 kernels, stride 2), producing 128-dim features. These are concatenated with message encoding (128 dims), BLStats encoding (155 dims), and previous action (16 dims), then projected to 256 dimensions.

\paragraph{Action space.}
11 discrete actions: 8 movement directions, \textsf{open\_door}, \textsf{kick}, and \textsf{search\_wait}. Typically 3--5 actions are valid depending on surrounding terrain.

\paragraph{Validity function.}
Implemented in \texttt{praxis/action\_masking/minihax\_corridor/masks.py}. \textsf{open\_door} requires adjacency to a closed (unlocked) door. Movement is invalid when blocked by walls or closed doors. \textsf{kick} is always valid but only necessary for locked doors.

Invalid actions that do not consume a game turn incur a configurable \texttt{frozen\_penalty} (default $-$0.01). Setting this to 0 yields pure no-op behavior.

\paragraph{Reward and termination.}
+1.0 for reaching the down staircase, 0 otherwise. Episodes terminate on reaching the stair, death, or timeout (1000 steps).

\subsection{Summary}

Table~\ref{tab:env_summary} summarizes key environment properties.

\begin{table}[h]
\centering
\caption{Environment summary.}
\label{tab:env_summary}
\small
\begin{tabular}{lccccc}
\toprule
\textbf{Environment} & $|\mathcal{A}|$ & \textbf{Valid/state} & \textbf{Validity ratio} & \textbf{Observation} & \textbf{Steps} \\
\midrule
Craftax-Full & 43 & 5--8 & 12--19\% & Symbolic (8268-d) & $10^9$ \\
Craftax-Classic & 17 & 5--8 & 29--47\% & Symbolic (1345-d) & $10^9$ \\
Corridor-2/3/5 & 11 & 3--5 & 27--45\% & GlyphNet (256-d) & $2 \times 10^8$ \\
\bottomrule
\end{tabular}
\end{table}

%% file: appendix_sections/appendix_validity.tex
\section{Validity Functions}
\label{app:validity}

We implement ground-truth action validity functions for each environment by extracting symbolic state features from the simulator and evaluating preconditions that mirror the environment's internal game logic. These functions are used both for oracle masking during training and for supervising the feasibility classification head. All implementations use JAX for vectorized evaluation across parallel environments.

\subsection{Craftax (43 Actions)}

Craftax exposes a rich symbolic state including inventory counts, nearby block types, player intrinsics (health, mana, energy), floor progression, and cardinal block adjacency. Each action's validity depends on specific preconditions:

\begin{itemize}[nosep]
    \item \textbf{Movement} (4 actions): blocked by solid blocks, water, or lava in the target direction.
    \item \textbf{Crafting} (13 actions): requires specific inventory items and proximity to crafting stations (table, furnace). Higher-tier items require multiple stations. Crafting is blocked if the player already has that tier or higher.
    \item \textbf{Floor transitions} (2 actions): \textsf{descend} requires standing on a down ladder and having killed $\geq 8$ monsters on the current floor; \textsf{ascend} requires standing on an up ladder and floor $> 0$.
    \item \textbf{Spells} (2 actions): require the spell to be learned and sufficient mana ($\geq 2$).
    \item \textbf{Consumables} (7 actions): potions require $\geq 1$ of that potion type; books require $\geq 1$ book.
    \item \textbf{Enchanting} (3 actions): requires an enchantment table, sufficient mana ($\geq 9$), a gem (ruby or sapphire), and the item to enchant.
    \item \textbf{Level-up} (3 actions): requires $\geq 1$ XP and the attribute below maximum (5).
    \item \textbf{Placement} (5 actions): requires the item in inventory; torches additionally require a dark floor.
    \item \textbf{Sleep}: only valid when energy $\leq 5$.
    \item \textbf{Sleep/rest lock}: while sleeping or resting, only \textsf{noop} is valid.
\end{itemize}

Listing~\ref{lst:craftax_mask} shows the implementation. A separate symbolic state extractor reads the raw environment state into a structured NamedTuple containing all fields needed for mask evaluation.

\begin{figure}[h]
\begin{lstlisting}[style=python, caption={Craftax validity function (43 actions). Each rule mirrors preconditions from \texttt{craftax/game\_logic.py}. The function operates on batched symbolic states and returns a boolean mask.}, label=lst:craftax_mask]
def craftax_action_mask(sym, num_actions=43):
    num_envs = sym.wood.shape[0]
    mask = jnp.ones((num_envs, num_actions), dtype=jnp.bool_)

    # Sleep only when low energy
    mask = mask.at[:, SLEEP].set(sym.player_energy <= 5)

    ct = sym.has_crafting_table
    fn = sym.has_furnace
    ct_fn = ct & fn

    # Placement actions
    mask = mask.at[:, PLACE_STONE].set(sym.stone >= 1)
    mask = mask.at[:, PLACE_TABLE].set(sym.wood >= 2)
    mask = mask.at[:, PLACE_FURNACE].set(sym.stone >= 1)
    mask = mask.at[:, PLACE_PLANT].set(sym.sapling >= 1)
    dark_floor = (sym.player_level == 2) | (sym.player_level == 5)
                 | (sym.player_level == 7) | (sym.player_level == 8)
    mask = mask.at[:, PLACE_TORCH].set((sym.torches >= 1) & dark_floor)

    # Pickaxes (require crafting table; iron/diamond need furnace)
    mask = mask.at[:, MAKE_WOOD_PICKAXE].set(
        ct & (sym.wood >= 1) & (sym.pickaxe < 1))
    mask = mask.at[:, MAKE_STONE_PICKAXE].set(
        ct & (sym.wood >= 1) & (sym.stone >= 1) & (sym.pickaxe < 2))
    mask = mask.at[:, MAKE_IRON_PICKAXE].set(
        ct_fn & (sym.wood >= 1) & (sym.stone >= 1) & (sym.iron >= 1)
        & (sym.coal >= 1) & (sym.pickaxe < 3))
    mask = mask.at[:, MAKE_DIAMOND_PICKAXE].set(
        ct & (sym.wood >= 1) & (sym.diamond >= 3) & (sym.pickaxe < 4))

    # Swords
    mask = mask.at[:, MAKE_WOOD_SWORD].set(
        ct & (sym.wood >= 1) & (sym.sword < 1))
    mask = mask.at[:, MAKE_STONE_SWORD].set(
        ct & (sym.wood >= 1) & (sym.stone >= 1) & (sym.sword < 2))
    mask = mask.at[:, MAKE_IRON_SWORD].set(
        ct_fn & (sym.wood >= 1) & (sym.stone >= 1) & (sym.iron >= 1)
        & (sym.coal >= 1) & (sym.sword < 3))
    mask = mask.at[:, MAKE_DIAMOND_SWORD].set(
        ct & (sym.wood >= 1) & (sym.diamond >= 2) & (sym.sword < 4))

    # Armour
    mask = mask.at[:, MAKE_IRON_ARMOUR].set(
        ct_fn & (sym.iron >= 3) & (sym.coal >= 3)
        & ((sym.armour < 1).sum(axis=-1) > 0))
    mask = mask.at[:, MAKE_DIAMOND_ARMOUR].set(
        ct & (sym.diamond >= 3) & ((sym.armour < 2).sum(axis=-1) > 0))

    # Ranged / consumables
    mask = mask.at[:, SHOOT_ARROW].set((sym.arrows >= 1) & (sym.bow >= 1))
    mask = mask.at[:, MAKE_ARROW].set(
        ct & (sym.wood >= 1) & (sym.stone >= 1) & (sym.arrows < 99))
    mask = mask.at[:, MAKE_TORCH].set(
        ct & (sym.wood >= 1) & (sym.coal >= 1) & (sym.torches < 99))

    # Spells
    mask = mask.at[:, CAST_FIREBALL].set(
        sym.learned_spells[:, 0] & (sym.player_mana >= 2))
    mask = mask.at[:, CAST_ICEBALL].set(
        sym.learned_spells[:, 1] & (sym.player_mana >= 2))

    # Potions
    for i in range(6):
        mask = mask.at[:, DRINK_POTION_RED + i].set(sym.potions[:, i] >= 1)

    # Books, enchanting
    mask = mask.at[:, READ_BOOK].set(sym.books >= 1)
    has_gem = (sym.ruby >= 1) | (sym.sapphire >= 1)
    can_enchant = (sym.has_enchant_fire | sym.has_enchant_ice) \
                  & (sym.player_mana >= 9) & has_gem
    mask = mask.at[:, ENCHANT_SWORD].set(can_enchant & (sym.sword >= 1))
    mask = mask.at[:, ENCHANT_ARMOUR].set(
        can_enchant & (sym.armour.sum(axis=-1) > 0))
    mask = mask.at[:, ENCHANT_BOW].set(can_enchant & (sym.bow >= 1))

    # Floor transitions
    current_kills = sym.monsters_killed[jnp.arange(num_envs), sym.player_level]
    mask = mask.at[:, DESCEND].set(
        sym.on_down_ladder & (current_kills >= 8))
    mask = mask.at[:, ASCEND].set(
        sym.on_up_ladder & (sym.player_level > 0))

    # Level up (xp >= 1, attribute < max)
    can_level = sym.player_xp >= 1
    mask = mask.at[:, LEVEL_UP_DEXTERITY].set(
        can_level & (sym.player_dexterity < 5))
    mask = mask.at[:, LEVEL_UP_STRENGTH].set(
        can_level & (sym.player_strength < 5))
    mask = mask.at[:, LEVEL_UP_INTELLIGENCE].set(
        can_level & (sym.player_intelligence < 5))

    # Movement: mask walking into solid blocks, water, or lava
    mask = mask.at[:, LEFT].set(~is_blocked(sym.block_left))
    mask = mask.at[:, RIGHT].set(~is_blocked(sym.block_right))
    mask = mask.at[:, UP].set(~is_blocked(sym.block_up))
    mask = mask.at[:, DOWN].set(~is_blocked(sym.block_down))

    # Sleep/rest lock: only NOOP while sleeping or resting
    locked = (sym.is_sleeping | sym.is_resting)[:, None]
    noop_only = jnp.zeros((num_envs, num_actions), dtype=jnp.bool_)
    noop_only = noop_only.at[:, NOOP].set(True)
    mask = jnp.where(locked, noop_only, mask)
    return mask
\end{lstlisting}
\end{figure}

\subsection{Craftax-Classic (17 Actions)}

Craftax-Classic uses a reduced action space without floor transitions, ranged combat, spells, enchanting, or potions. The validity function follows the same structure but is simpler: crafting requires a crafting table (and furnace for iron-tier items), placement requires the item in inventory, and movement is blocked only by lava tiles. The \textsf{do} action (interact with facing block) is valid only when the facing tile is interactable (tree, ore, water, ripe plant, grass, or mob). Sleep locks all actions to \textsf{noop}.

\begin{figure}[h]
\begin{lstlisting}[style=python, caption={Craftax-Classic validity function (17 actions).}, label=lst:classic_mask]
def craftax_classic_action_mask(sym):
    num_envs = sym.wood.shape[0]
    mask = jnp.ones((num_envs, 17), dtype=jnp.bool_)

    ct = sym.has_crafting_table
    fn = sym.has_furnace
    ct_fn = ct & fn

    # Placement
    mask = mask.at[:, PLACE_STONE].set(sym.stone >= 1)
    mask = mask.at[:, PLACE_TABLE].set(sym.wood >= 2)
    mask = mask.at[:, PLACE_FURNACE].set(sym.stone >= 1)
    mask = mask.at[:, PLACE_PLANT].set(sym.sapling >= 1)

    # Pickaxes
    mask = mask.at[:, MAKE_WOOD_PICKAXE].set(
        ct & (sym.wood >= 1) & ~sym.wood_pickaxe)
    mask = mask.at[:, MAKE_STONE_PICKAXE].set(
        ct & (sym.wood >= 1) & (sym.stone >= 1) & ~sym.stone_pickaxe)
    mask = mask.at[:, MAKE_IRON_PICKAXE].set(
        ct_fn & (sym.wood >= 1) & (sym.stone >= 1)
        & (sym.iron >= 1) & (sym.coal >= 1) & ~sym.iron_pickaxe)

    # Swords
    mask = mask.at[:, MAKE_WOOD_SWORD].set(
        ct & (sym.wood >= 1) & ~sym.wood_sword)
    mask = mask.at[:, MAKE_STONE_SWORD].set(
        ct & (sym.wood >= 1) & (sym.stone >= 1) & ~sym.stone_sword)
    mask = mask.at[:, MAKE_IRON_SWORD].set(
        ct_fn & (sym.wood >= 1) & (sym.stone >= 1)
        & (sym.iron >= 1) & (sym.coal >= 1) & ~sym.iron_sword)

    # Movement: block walking into lava
    mask = mask.at[:, LEFT].set(sym.block_left != LAVA)
    mask = mask.at[:, RIGHT].set(sym.block_right != LAVA)
    mask = mask.at[:, UP].set(sym.block_up != LAVA)
    mask = mask.at[:, DOWN].set(sym.block_down != LAVA)

    # Do: only when facing something interactable
    fb = sym.facing_block
    can_do = (sym.facing_has_mob | (fb == TREE) | (fb == WATER)
              | ((fb == STONE) & sym.wood_pickaxe)
              | ((fb == COAL) & sym.wood_pickaxe)
              | ((fb == IRON) & sym.stone_pickaxe)
              | ((fb == DIAMOND) & sym.iron_pickaxe)
              | (fb == RIPE_PLANT) | (fb == GRASS))
    mask = mask.at[:, DO].set(can_do)

    # Sleep lock
    locked = sym.is_sleeping[:, None]
    noop_only = jnp.zeros((num_envs, 17), dtype=jnp.bool_)
    noop_only = noop_only.at[:, NOOP].set(True)
    mask = jnp.where(locked, noop_only, mask)
    return mask
\end{lstlisting}
\end{figure}

\subsection{MiniHack Corridor-5 (11 Actions)}

MiniHack Corridor-5 uses a tile-based grid with 8-directional movement, doors, and search/wait. The symbolic state consists of the 8 adjacent tile types. Movement is valid only if the target tile is walkable (floor, stairs, corridor, or open door). \textsf{open\_door} is valid only when adjacent to a closed door. \textsf{kick} is valid only when adjacent to a locked door. \textsf{search\_wait} is always valid.

\begin{figure}[h]
\begin{lstlisting}[style=python, caption={MiniHack Corridor-5 validity function (11 actions).}, label=lst:corridor_mask]
WALKABLE = [FLOOR, UPSTAIR, DOWNSTAIR, CORRIDOR, DOOR_OPEN]

def corridor_action_mask(sym, num_actions=11):
    num_envs = sym.adjacent_tiles.shape[0]
    mask = jnp.ones((num_envs, num_actions), dtype=jnp.bool_)

    # Movement (0-7): allow if target tile is walkable
    walkable = jnp.isin(sym.adjacent_tiles, jnp.array(WALKABLE))
    mask = mask.at[:, :8].set(walkable)

    # OPEN_DOOR (8): adjacent to a closed door
    has_closed_door = jnp.any(
        sym.adjacent_tiles == DOOR_CLOSED, axis=-1)
    mask = mask.at[:, OPEN_DOOR].set(has_closed_door)

    # KICK (9): adjacent to a locked door
    has_locked_door = jnp.any(
        sym.adjacent_tiles == DOOR_LOCKED, axis=-1)
    mask = mask.at[:, KICK].set(has_locked_door)

    # SEARCH_WAIT (10) always allowed
    return mask
\end{lstlisting}
\end{figure}

%% file: appendix_sections/appendix_training.tex
\section{Training Details}
\label{app:training}

\subsection{Architectural Details}

\paragraph{PPO-Hybrid (Transformer-XL + S5).}
Observations are projected to 256 dimensions via a linear layer (Craftax) or GlyphNet encoder (MiniHack). The backbone uses one Transformer-XL layer followed by five S5 layers (pattern: TSSSSS), each with model dimension 256 and SSM state dimension 256. The Transformer layer uses 8 attention heads with QKV dimension 256, gated residuals, and memory window of 128 tokens. Actor and critic heads are 2-layer MLPs with 256 hidden units and ReLU activations.

\paragraph{PPO-RNN (GRU).}
Observations are embedded to 512 dimensions, then passed through a single-layer GRU with 512 hidden units. Actor and critic heads are 2-layer MLPs with 512 hidden units and ReLU activations.

\paragraph{PPO-MLP.}
Separate actor and critic trunks, each with three dense layers (512 hidden units, tanh activations). Used on Craftax only; symbolic observations provide sufficient information without memory.

\subsection{Hyperparameters}

Table~\ref{tab:hyperparams} reports PPO hyperparameters. Parameters that differ across architectures are marked with $^*$.

\begin{table}[h]
\centering
\caption{PPO hyperparameters. $^*$ indicates architecture-specific values.}
\label{tab:hyperparams}
\small
\begin{tabular}{lccc}
\toprule
\textbf{Hyperparameter} & \textbf{PPO-Hybrid} & \textbf{PPO-RNN} & \textbf{PPO-MLP} \\
\midrule
Parallel environments & 1024 & 1024 & 1024 \\
Steps per rollout & 128$^*$ & 64$^*$ & 64$^*$ \\
Batch size & 131{,}072 & 65{,}536 & 65{,}536 \\
Number of minibatches & 8 & 8 & 8 \\
Minibatch size & 16{,}384 & 8{,}192 & 8{,}192 \\
Update epochs & 4 & 4 & 4 \\
Learning rate & $2 \times 10^{-4}$ & $2 \times 10^{-4}$ & $2 \times 10^{-4}$ \\
LR annealing & Linear & Linear & Linear \\
Discount $\gamma$ & 0.999$^*$ & 0.99$^*$ & 0.99$^*$ \\
GAE $\lambda$ & 0.8 & 0.8 & 0.8 \\
Clip ratio $\epsilon$ & 0.2 & 0.2 & 0.2 \\
Entropy coefficient & 0.002$^*$ & 0.01$^*$ & 0.01$^*$ \\
Value function coefficient & 0.5 & 0.5 & 0.5 \\
Max gradient norm & 1.0 & 1.0 & 1.0 \\
Activation & ReLU & Tanh & Tanh \\
\midrule
\multicolumn{4}{l}{\textbf{Auxiliary losses (C3 and C4 only)}} \\
\midrule
Classification coefficient $\lambda$ & 10 & 10 & 10 \\
Focal parameter $\gamma_{\text{focal}}$ & 2.0 & 2.0 & 2.0 \\
\midrule
\multicolumn{4}{l}{\textbf{RND (MiniHack only)}} \\
\midrule
RND reward coefficient & 1.0 & 1.0 & 1.0 \\
RND loss coefficient & 0.01 & 0.01 & 0.01 \\
RND intrinsic GAE coefficient & 0.01 & 0.01 & 0.01 \\
RND network hidden / output & 256 / 512 & 256 / 512 & 256 / 512 \\
\bottomrule
\end{tabular}
\end{table}

Batch size equals parallel environments $\times$ steps per rollout. Minibatch size equals batch size $/$ number of minibatches.

\subsection{Training Infrastructure}

All experiments use 1024 parallel environments on a single accelerator. Craftax: $10^9$ steps. MiniHack: $2 \times 10^8$ steps. Results report mean and standard deviation over 4 seeds.

\subsection{Measurement of Suppression}

We instrument the environment to log $\pi(a \mid s^*)$ the first time each state-action pair $(s^*, a)$ becomes valid during training:

\begin{enumerate}
    \item \textbf{First-occurrence log-probability}: $\log \pi_T(a \mid s^*)$ at the step when $s^*$ first enters the training batch.
    \item \textbf{Time-to-valid}: steps between when $a$ first becomes valid and when $\pi(a \mid s^*) > 0.5$.
    \item \textbf{Suppression ratio}: first-occurrence probability divided by uniform initialization $1/|\mathcal{A}|$.
\end{enumerate}

Metrics are computed online during training by caching state-action validity transitions.

%% file: appendix_sections/appendix_ablations.tex
\section{Ablation Studies}
\label{app:ablations}




%% file: appendix_sections/appendix_additional_experiments.tex
\section{Additional Experiments}
\label{app:additional}

\subsection{Classification with Oracle Masking}
\label{app:mask_predictor_ablation}

The classification loss is designed for settings where oracle masks are unavailable, but it may also benefit masked agents by improving feature quality. We compare three conditions on Craftax with PPO-Hybrid. Condition (a), Unmasked + Predictor, uses KL-balanced classification with $\lambda = 20$ without oracle masking. Condition (b), Masked Only, uses oracle masking with no classification head. Condition (c), Masked + Predictor, combines oracle masking with KL-balanced classification. Figure~\ref{fig:mask_predictor_ablation} shows the results. All three conditions converge to similar episode return, though masked training learns faster early and the unmasked + predictor condition catches up by 400M frames. Both conditions with classification heads achieve 99\% validity prediction accuracy, so the encoder learns to distinguish valid and invalid states regardless of masking. When evaluated with oracle masks at test time, masked + predictor achieves higher return than masked only, indicating that explicitly training the encoder to separate valid and invalid states improves feature quality and policy performance even when masking handles the policy-level constraint. Evaluation with predicted masks is most practically relevant: masked only collapses to near-zero return because it lacks a classification head. Masked + predictor degrades to the performance of unmasked + predictor. Having oracle masks during training does not produce a better validity predictor than training with the classification loss alone.

Masking eliminates suppression at the policy level; classification improves representations at the feature level. Combining both yields the strongest performance under oracle evaluation, while the classification head alone provides robustness when oracle masks are unavailable.

\begin{figure}[h]
    \centering
    \includegraphics[width=\textwidth]{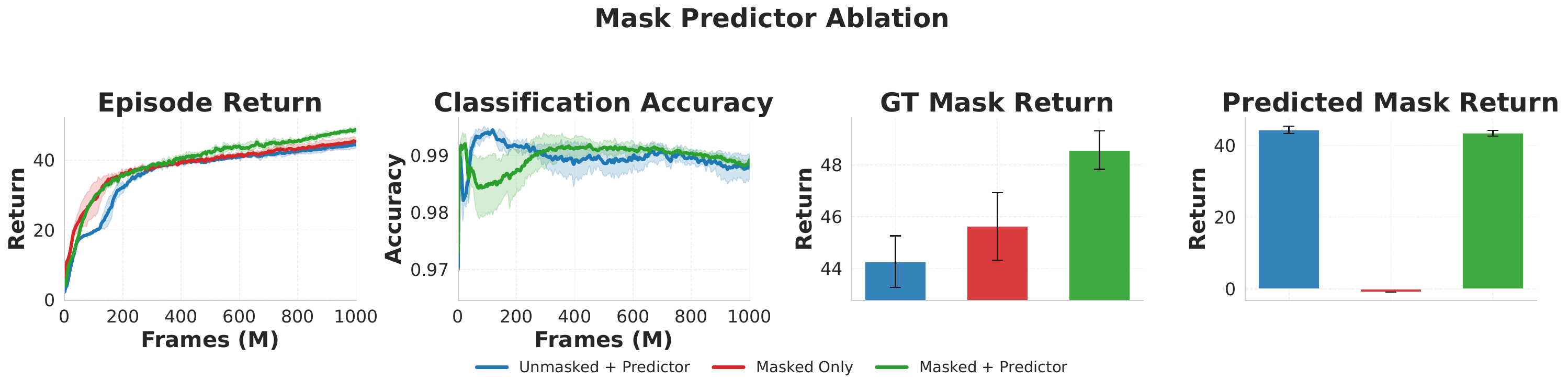}
    \caption{Mask predictor ablation on Craftax (PPO-Hybrid, $\lambda = 20$). Left: Episode return during training. Center-left: Classification accuracy. Center-right: Return with ground-truth masks. Right: Return with predicted masks.}
    \label{fig:mask_predictor_ablation}
\end{figure}

\subsection{Classification Coefficient Sweep (Masked + KL)}
\label{app:coeff_sweep_masked_kl}

We sweep the classification coefficient $\lambda \in \{0.1, 1, 5, 10, 20\}$ for the KL-balanced loss applied on top of oracle masking on Craftax (PPO-Hybrid). This identifies how much classification signal is needed to improve masked training. Figure~\ref{fig:coeff_sweep_masked_kl} shows the results.

Episode return is remarkably stable across the sweep. All coefficients converge to similar final return near 43--46, indicating that the KL-balanced classification loss does not interfere with policy optimization even at high coefficients. Higher classification coefficients ($\lambda \geq 5$) drive accuracy above 98\% and classification loss below 0.2, while $\lambda = 0.1$ achieves only approximately 97\% accuracy with higher residual loss. Sufficient classification signal is needed for the encoder to reliably separate valid and invalid states.

Under ground-truth mask evaluation, higher coefficients ($\lambda \in \{5, 10, 20\}$) produce small but statistically significant improvements over lower coefficients, consistent with Section~\ref{app:mask_predictor_ablation}: the classification loss improves feature quality even when oracle masks are available. Under predicted mask evaluation, the effect is more pronounced. Coefficients $\lambda \in \{5, 10, 20\}$ achieve predicted-mask return within 5\% of ground-truth-mask return, while $\lambda = 0.1$ and $\lambda = 1$ show a larger gap. The classification head must be sufficiently accurate for the predicted masks to serve as a viable substitute for oracle masks at deployment time.

We use $\lambda = 10$ by default because it achieves near-ceiling classification accuracy without destabilizing policy learning and produces predicted masks accurate enough to replace oracle masks at negligible performance cost.

\begin{figure}[h]
    \centering
    \includegraphics[width=\textwidth]{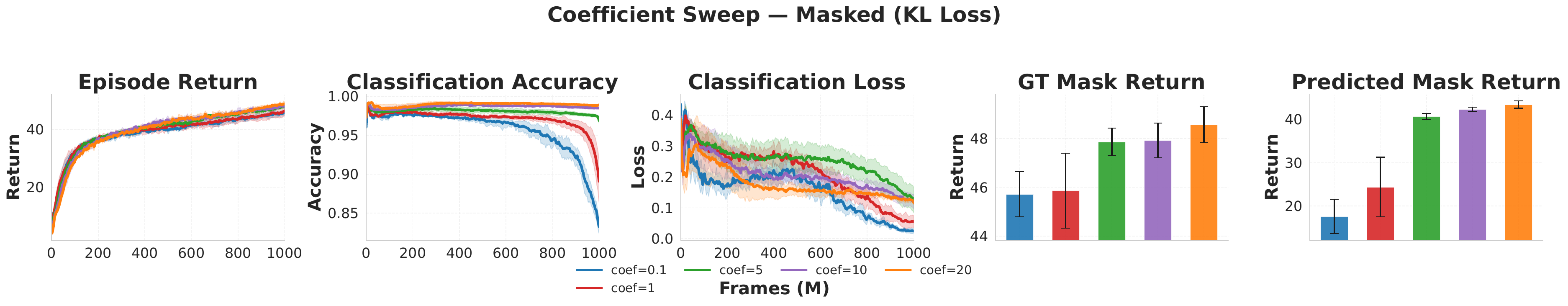}
    \caption{Classification coefficient sweep with oracle masking and KL-balanced loss on Craftax (PPO-Hybrid). Left: Episode return. Center-left: Classification accuracy. Center-right: Classification loss. Right: Final return evaluated with ground-truth masks (dark) and predicted masks (light).}
    \label{fig:coeff_sweep_masked_kl}
\end{figure}

\subsection{Suppression and Recovery of the Descend Action}
\label{app:descend_suppression}

The descend action in Craftax is valid only when the agent stands on a staircase (rare early in training) but is critical for dungeon progression. Figure~\ref{fig:descend_early} tracks probability and feature correlation for descend across the first 200M frames.

Without masking, $\pi(\text{descend} \mid s^*)$ drops from uniform initialization $1/43 \approx 0.023$ to below $10^{-3}$ within 100M frames. This drop exceeds an order of magnitude, matching Theorem~\ref{thm:prob_suppression}. The agent has not visited staircase states frequently enough for the classification loss to build separation. Around 125--150M frames, unmasked + predictor recovers sharply. Probability jumps back above $10^{-1}$ as feature correlation drops from approximately 0.75 to approximately 0.3. Once features separate, suppression breaks and the policy rapidly reallocates probability. This delay corresponds to the time needed to accumulate gradient signal on rare valid states.

The masked conditions never experience suppression. $\pi(\text{descend} \mid s^*)$ stays above $10^{-1}$ throughout. But they differ in feature quality. Masked only maintains high feature correlation near 0.75--0.85 and high probability at invalid states near 0.3--0.5, relying entirely on masking for correct behavior. Masked + predictor achieves lower correlation near 0.65--0.7 and lower invalid-state probability, which indicates classification improves representations even when masking prevents suppression at the policy level.

\begin{figure}[h]
    \centering
    \includegraphics[width=\textwidth]{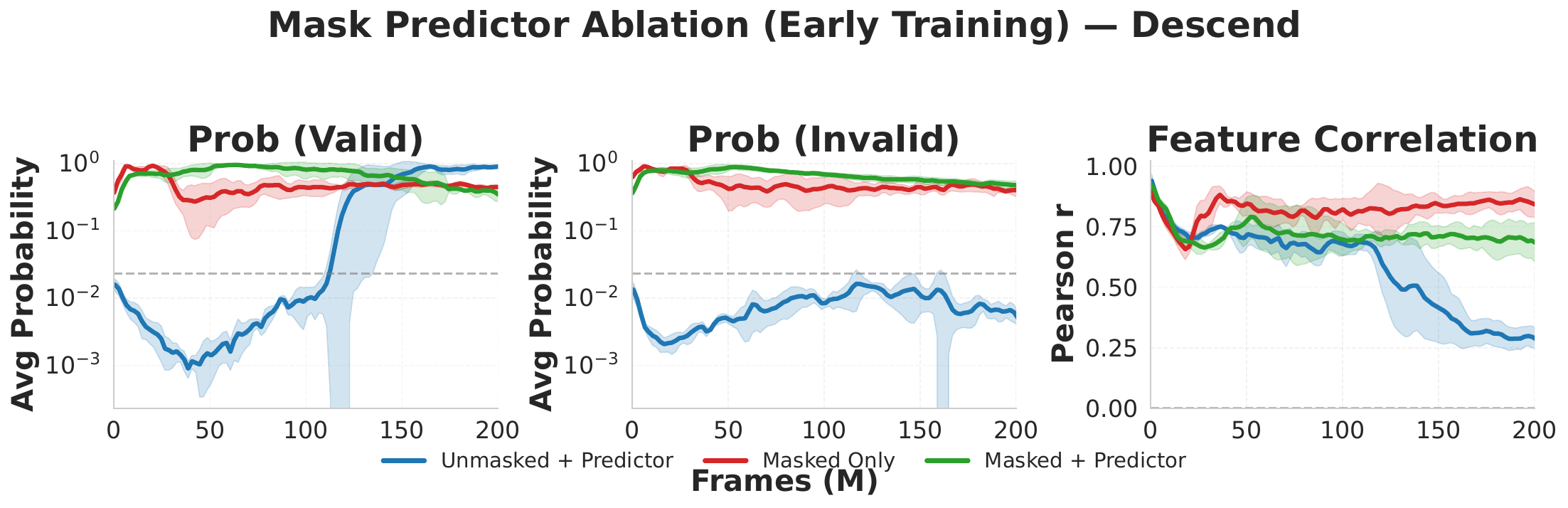}
    \caption{Suppression profile of the descend action during early training (PPO-Hybrid, $\lambda = 20$). Left: probability at valid states (log scale). Center: probability at invalid states (log scale). Right: Pearson correlation between feature representations of valid and invalid states.}
    \label{fig:descend_early}
\end{figure}